\documentclass[11pt]{article}
\usepackage[numbers, compress]{natbib}
\usepackage[margin=1in]{geometry}
\usepackage{subcaption}

\usepackage[table]{xcolor}  
\definecolor{Gray}{gray}{0.9}
\definecolor{midgreen}{rgb}{0.1,0.5,0.1}
\definecolor{darkgray}{gray}{0.25}
\definecolor{lightblue}{rgb}{0.25,0.25,0.8}
\definecolor{mydarkblue}{rgb}{0,0.08,0.45}
\usepackage[colorlinks,
linkcolor=lightblue,
anchorcolor=blue,
urlcolor=mydarkblue,
citecolor=lightblue]{hyperref}
\usepackage{makebox}
\usepackage{amsmath}
\usepackage{amssymb}
\usepackage{mathtools}
\usepackage{amsthm}
\usepackage{setspace}
\usepackage{bm}
\usepackage{bbm}

\newcommand{\norm}[1]{\ensuremath{\left\| #1 \right\|}}

\newcommand{\abs}[1]{\left |#1\right|}

\def\0{{\bm 0}}

\def\c{{\bm c}}

\def\e{{\bm e}}

\def\k{{\bm k}}

\def\q{{\bm q}}
\def\r{{\bm r}}
\def\s{{\bm s}}

\def\v{{\bm v}}

\def\x{{\bm x}}
\def\y{{\bm y}}
\def\z{{\bm z}}

\def\K{{\bm K}}

\def\M{{\bm M}}

\def\S{{\bm S}}

\def\V{{\bm V}}

\def\BPi{\boldsymbol{\Pi}}

\def\E{\mathop{{\mathbb{E}}}}

\def\Ccal{\mathcal{C}}

\def\Ncal{\mathcal{N}}

\def\RR{\mathbb{R}}

\def\SS{\mathbb{S}}

\newtheorem{theorem}{Theorem}
\newtheorem{lemma}{Lemma}

\newtheorem{definition}{Definition}

\def\llama{$\mathtt{Llama}$-$\mathtt{3.1}$-$\mathtt{8B}$-$\mathtt{Instruct}$}
\def\mistral{$\mathtt{Ministral}$-$\mathtt{7B}$-$\mathtt{Instruct}$}

\def\whK{\widehat{\makebox*{$m$}{$ \K $}}}
\def\whV{\widehat{\makebox*{$m$}{$ \V $}}}

\def\bfpn{b_{\text{FPN}}}

\usepackage{xspace}
\newcommand{\TurboQuant}{\textsc{TurboQuant}\xspace}

\newcommand{\email}[1]{\href{mailto:#1}{\color{black} \texttt{#1}}}

\title{TurboQuant: Online Vector Quantization with Near-optimal Distortion Rate}

\author{
  Amir Zandieh\\
  Google Research\\
  \email{zandieh@google.com} \\
  \and
  Majid Daliri\\
  New York University\\
  \email{daliri.majid@nyu.edu} \\
  \and 
  Majid Hadian\\
  Google DeepMind\\
  \email{majidh@google.com} \\
  \and 
  Vahab Mirrokni \\
  Google Research\\
  \email{mirrokni@google.com} \\
}

\date{}

\usepackage[capitalize]{cleverref}
\usepackage{algorithm}
\usepackage{algorithmic}
\usepackage{amsthm}

\usepackage{multirow,makecell}
\usepackage{booktabs}

\usepackage{tikz}
\usetikzlibrary{matrix}
\usetikzlibrary{positioning}

\begin{document}

\maketitle

\begin{abstract}
Vector quantization, a problem rooted in Shannon's source coding theory, aims to quantize high-dimensional Euclidean vectors while minimizing distortion in their geometric structure. 
We propose \TurboQuant to address both mean-squared error (MSE) and inner product distortion, overcoming limitations of existing methods that fail to achieve optimal distortion rates. 
Our data-oblivious algorithms, suitable for online applications, achieve near-optimal distortion rates (within a small constant factor) across all bit-widths and dimensions. 
\TurboQuant achieves this by randomly rotating input vectors, inducing a concentrated Beta distribution on coordinates, and leveraging the near-independence property of distinct coordinates in high dimensions to simply apply optimal scalar quantizers per each coordinate. 
Recognizing that MSE-optimal quantizers introduce bias in inner product estimation, we propose a two-stage approach: applying an MSE quantizer followed by a 1-bit Quantized JL (QJL) transform on the residual, resulting in an unbiased inner product quantizer. 
We also provide a formal proof of the information-theoretic lower bounds on best achievable distortion rate by any vector quantizer, demonstrating that \TurboQuant closely matches these bounds, differing only by a small constant ($\approx 2.7$) factor.
Experimental results validate our theoretical findings, showing that for KV cache quantization, we achieve absolute quality neutrality with 3.5 bits per channel and marginal quality degradation with 2.5 bits per channel. 
Furthermore, in nearest neighbor search tasks, our method outperforms existing product quantization techniques in recall while reducing indexing time to  virtually zero.

\end{abstract}

\section{Introduction}
\label{intro}

Vector quantization (VQ) in Euclidean space is crucial for efficiently handling high-dimensional vectors across a spectrum of computational domains, from training and deploying large-scale AI and deep learning models to powering vector databases for search/retrieval systems. 
The core objective is to compress high dimensional vectors by quantizing them--converting floating-point coordinate values to low-bitwidth integers--while minimizing distortion, quantified by metrics such as mean-squared error (MSE) or inner product errors.
By preserving these properties, inner product queries can be answered rapidly, with minimal latency, and using reduced computational and communication resources.

This problem's roots trace back to Shannon’s seminal work on Source Coding theory~\cite{shannon1948mathematical, shannon1959coding}, which established that the least distortion achievable by block source codes, now known as vector quantizers, is defined by the Shannon distortion-rate function, determined by the statistical properties of the source and the chosen distortion measure, such as MSE. 
Today, VQ plays a critical role in fundamental computational domains, including AI, deep learning, and search systems.

A key application of VQ is in the deployment of AI models, including large language models (LLMs)~\cite{achiam2023gpt, dubey2024llama, claude, team2024gemini}. 
As LLM capabilities depend heavily on their model size and context length~\cite{kaplan2020scaling}, serving them requires substantial memory demands and increased inference latency. 
This latency is primarily attributed to communication bottlenecks between HBM and SRAM on accelerators, or across distributed clusters. 
By compressing or quantizing model weights and activations, we can effectively mitigate these bottlenecks, resulting in significant reductions in inference costs. 
Inner product operations between activations and weights is at the core of deep learning models. 
Thus, model quantization schemes strive to compress weights and/or activation vectors while accurately preserving these inner products.

Decoder based transformer models~\cite{vaswani2017attention} present another compelling use case.
These models must store key/value (KV) embeddings from previously generated tokens in the KV cache, the size of which scales with both model size (number of layers and attention heads) and context length. 
This scaling is a significant bottleneck in terms of memory usage and computational speed, especially for long context models. 
Therefore, reducing the KV cache size without compromising accuracy is essential.
In this context, the preservation of the Euclidean structure of these embedding vectors--their inner products and distances--is crucial for maintaining model performance. 
VQ emerges as the most suitable framework for addressing this challenge, offering a robust approach to compressing high-dimensional embeddings while preserving their essential geometric properties.

Additionally, nearest neighbor (NN) search in high-dimensional spaces with inner product or cosine similarity~\cite{elasticsearch, guo2020accelerating} is a cornerstone of vector databases~\cite{pinecone, gdrant, pgvector}.
These databases are fundamental for retrieval-augmented generation~\cite{gao2023retrieval, edge2024local} and information retrieval~\cite{khattab2020colbert, santhanam2021colbertv2}. 
VQ, a.k.a. product quantization (PQ), plays a critical role in these applications. 
It enables efficient compression of database vectors, optimizes memory usage, and facilitates low-latency, accurate estimations of inner products with query vectors, thereby enabling fast and precise nearest neighbor searches.

Existing VQ algorithms present a trade-off: either they lack accelerator (vectorization) compatibility and exhibit slow computation, making them unsuitable for real-time AI applications like KV cache quantization, or they suffer from suboptimal distortion bounds relative to bit-width.
Our objective is to introduce an algorithm that addresses these limitations.
Specifically, we design \TurboQuant: a lightweight, capable of online application (crucial for scenarios like KV cache quantization), and highly accelerator-friendly—a critical attribute for modern AI workloads. 

The core of \TurboQuant is a two-stage process. 
First, we develop a vector quantizer with optimal distortion rate in terms of mean-squared error (MSE). 
Subsequently, we apply a 1-bit quantizer to the residual, resulting in an unbiased and low-distortion inner product quantizer. 
We demonstrate that quantizers optimized for MSE do not produce unbiased estimators for inner products, and our two-stage solution effectively bridges this gap.
Our MSE-optimal quantizer starts by randomly rotating $d$-dimensional input vectors. 
Observing the key fact that each coordinate in the rotated vectors follows a Beta distribution, we design optimal Lloyd-Max quantizer~\cite{lloyd1982least, max1960quantizing} for each coordinate by solving a continuous k-means problem.
This method gives optimal MSE distortion bound and minimizes the L2 norm of the residual.
To obtain an unbiased and low-distortion quantizer for inner products, we compose our quantizer with the recently developed Quantized Johnson-Lindenstrauss (QJL) transform~\cite{qjl}, which quantizes each coordinate of the residual vector to a single bit.
Our algorithm offers provably optimal distortion bounds for both MSE and inner products, achieving an exponential improvement over existing methods in terms of bit-width dependence.

\subsection{Problem Definition}

Formally, our goal is to design a quantization map, denoted as $Q: \RR^d \to \{ 0, 1 \}^B$, that transforms $d$-dimensional vectors to a binary string of $B$ bits.
If we set $B = b \cdot d$ for some $b \ge 0$, this quantizer will have a bit-width of $b$, representing the average number of bits used to encode each real-valued coordinate of $\RR^d$.
Crucially, we require an inverse map, $Q^{-1}: \{ 0, 1 \}^B \to \RR^d$ that performs dequantization, approximately reconstructing original vectors from their quantized representations.
Of course, this transformation is inherently lossy, as $Q$ is not a bijection. 
So, our primary objective is to minimize distortion, with a specific focus on mean-squared error (MSE) and inner product distortion.

We make no assumptions about the input vector dataset, considering the worst-case scenario. 
We let the quantizer $Q(\cdot)$ to be randomized, leading to stochastic outputs.
Considering randomized quantizers, it is more appropriate to define the expected distortion over the randomness of the quantizer's output. 
Thus, we aim to design quantizers that for any desired bit-width $b$ minimize the following expected distortion measures for any (worst-case) vectors $\x, \y \in \RR^d$:
\begin{align}
    \textbf{(MSE)}~~~~~ & D_{\tt mse} := \E_{Q}\left[\left\| \x - Q^{-1}\left( Q(\x) \right) \right\|_2^2 \right] \label{eq:mse}\\
    \textbf{(inner-prod error)}~~~~~ & D_{\tt prod} := \E_{Q}\left[\left| \langle \y, \x \rangle - \langle \y, Q^{-1}\left( Q(\x) \right) \rangle \right|^2 \right]. \label{eq:prod_error}
\end{align}
The expectations above are takes with respect to the randomness of the quantizer $Q(\cdot)$.
Furthermore, for inner-product quantizers, we require unbiasedness of the inner product estimator, a desirable property for numerous applications.
More precisely, we require:
\begin{align*}
    \textbf{(unbiased inner-prod)}~~~~~ &\E_{Q}\left[ \langle \y, Q^{-1}\left( Q(\x) \right) \rangle\right] = \langle \y, \x \rangle.
\end{align*}

We aim to design computationally efficient quantizers $Q_{\tt mse}$ and $Q_{\tt prod}$, that achieve optimal bounds for the distortion measures defined above, for any given bit-width $b$. 
Additionally, we aim for $Q_{\tt prod}$ to provide unbiased inner product estimates. 
In particular, assume that we are given $n$ real-valued vectors $x_1, x_2, \ldots x_n \in \RR^d$. We design the following primitives:
\begin{itemize}
    \item \textsc{Quant}: efficiently quantizes the dataset and computes $Q(\x_1), Q(\x_2), \ldots Q(\x_n)$. 
    \item \textsc{DeQuant}: given a quantized dataset, can efficiently reconstruct original vectors by computing $Q^{-1}\left( Q(\x_i) \right)$ for any $i \in [n]$.
\end{itemize}

\subsection{Related Work}

\paragraph{Beginnings of VQ.}
The vector quantization theory started by Shannon's seminal work~\cite{shannon1948mathematical, shannon1959coding} on achievable distortion-rate functions. 
In 1963, Zador~\cite{zador1964development} made significant advances by employing high-resolution methods to derive the limiting operational distortion-rate function for fixed-rate quantization at high rates that closely matches Shannon's distortion-rate function. 
However, Zador did not specifically consider implementable algorithms. 
Gersho's influential paper~\cite{gersho1979asymptotically}, further advanced the vector quantization by popularizing high-resolution theory, simplifying Zador's results, introducing lattice vector quantization, and proposing a key conjecture that shaped the field.
Despite these theoretical advancements, the practical applicability of vector quantization remained unclear in early years. 
The most straightforward encoding method, brute-force nearest neighbor search, was computationally expensive, hindering the adoption of VQ in practice.

\paragraph{Online vs Offline Quantization.}
Online (data-oblivious) quantization methods apply instantly without needing data-specific tuning or calibrations~\cite{dettmers2022gpt3, ashkboos2024quarot, liu2024kivi, shah2024flashattention, han2025polarquant}. 
In contrast, offline (data-dependent) methods require heavy preprocessing and learning to adapt the quantization map to the data, making them unsuitable for dynamic data scenarios~\cite{kim2023squeezellm}.
For instance, methods such as those presented in~\cite{frantar2022gptq, lin2024awq, xiao2023smoothquant, chee2023quip} use second-order (Hessian) information to tune the quantization map which requires heavy preprocessing and even in some cases post processing as well.

\paragraph{Online KV Cache Compression.}

Several approaches have been proposed to compress the KV cache.
These include architectural modifications~\cite{shazeer2019fast, ainslie2023gqa, dai2024deepseekmoe} which restructure the transformer to minimize the number of stored key-value pairs.
Additionally, pruning or evicting redundant or less critical tokens has emerged as another approach~\cite{beltagy2020longformer, zhang2024h2o, liu2024scissorhands, xiao2023efficient, zandieh2024subgen, li2024snapkv, han2025balancekv}.

A simple yet effective approach to reducing KV cache size is quantizing the KV cache. 
Several quantization techniques have been developed specifically for this purpose~\cite{yue2024wkvquant, yang2024no, dong2024qaq, kang2024gear, zhang2024kv, liu2024kivi, hooper2024kvquant, kim2024lexico, han2025polarquant}.
Recently, a new quantization called QJL~\cite{qjl} introduced an efficient, data-oblivious 1-bit quantization approach based on sketching techniques, which provides unbiased estimates for inner product queries.
This method does not require tuning or adaptation to the input data and we make use of this technology in our quantizer optimized for inner product distortion.

\paragraph{Product Quantization (PQ).}
In Near Neighbor (NN) search problem with Euclidean datasets, the index size poses a significant memory bottleneck, often mitigated by quantization techniques, commonly referred to as Product Quantization (PQ) in the NN literature.
Many of these algorithms rely on constructing a quantization codebook using variations of k-means during the indexing phase~\cite{jegou2010product, babenko2014additive, ge2013optimized, wang2017survey, guo2020accelerating}. 
Therefore, these methods are ill-suited for online settings due to their requirement for extensive preprocessing.

Recently, a grid-based PQ method was introduced in \cite{gao2024practical}, eliminating the need for preprocessing. 
This approach operates by projecting a uniform grid onto the unit sphere and conducting a search to identify the nearest projection to the data points. 
While the paper's theoretical guarantees are suboptimal, likely due to loose analysis—as practical performance surpasses theoretical bounds—the grid projection and binary search algorithm is also computationally slow and particularly inefficient on accelerators like GPU because of their algorithm's inherent lack of vectorization, which prevents parallel processing.


\subsection{Overview of Techniques and Contributions}

\paragraph{MSE Optimzied \TurboQuant.}
Our first VQ algorithm is designed to minimize MSE distortion deinfed in \cref{eq:mse}.
To achieve this, we apply a random rotation to the input vectors, thereby inducing a Beta distribution on each coordinate, irrespective of the input vectors themselves.
In high dimensions $d$, the distribution of each coordinate converges to a Gaussian distribution $\Ncal(1, 1/d)$ due to concentration of measure and the central limit theorem. 
Furthermore, any two distinct coordinates become nearly uncorrelated and, more importantly, almost independent (a deeper result that goes beyond just correlation).
This near-independence is a crucial aspect that simplifies our quantization design.
It allows us to quantize each coordinate using optimal scalar quantization, disregarding interactions or correlations between different coordinates, while still achieving near-optimal distortion.

We find optimal scalar quantizers for random variables with Beta distributions by solving a continuous $1$-dimensional k-means problem using the Max-Lloyd algorithm. 
We precompute and store these optimal codebooks for a range of practically useful bit-widths, to enable efficient subsequent invocations of our \TurboQuant algorithm.

In \cref{thrm_mse} we prove that the $b$-bit MSE optimized \TurboQuant $Q_{\tt mse}: \RR^d \to \{ 0, 1 \}^{b \cdot d}$ achieves the following distortion for any worst-case vector $\x \in \RR^d$ with $\norm{\x}=1$:
\begin{itemize}
    \item $D_{\tt mse}(Q_{\tt mse}) := \E\left[\left\| \x - Q_{\tt mse}^{-1}\left( Q_{\tt mse}(\x) \right) \right\|_2^2 \right] \le \frac{\sqrt{3} \pi}{2} \cdot \frac{1}{4^{b}}$ for any $b \ge 0$.
    \item For small bit-widths the above distortion upper bound can be further refined. 
    Specifically, for $b = 1, 2, 3, 4$ we have $D_{\tt mse}(Q_{\tt mse}) \approx {\bf 0.36}, {\bf 0.117}, {\bf 0.03}, {\bf 0.009}$, respectively.
\end{itemize}

Note that the unit norm assumption, $\norm{x}_2=1$, is standard and not restrictive. 
For datasets that do not satisfy this assumption we can compute and store the $L2$ norms in floating-point precision and rescale the dequantized points using these stored norms.

\paragraph{Inner Product \TurboQuant.}
We show that the MSE optimized quantizers are biased for inner product estimation and thus a different VQ scheme is needed to get an unbiased inner product quantizer.
Our solution is a two stage algorithm that first applies the abovementioned $Q_{\tt mse}$ with a bit-width one less than our target budget and then apply a QJL~\cite{qjl} on the residual error.
This is proved to be unbiased and also has nearly optimal inner product error rate.

In \cref{thrm_prod} we prove that the $b$-bit inner product optimized \TurboQuant $Q_{\tt prod}: \RR^d \to \{ 0, 1 \}^{b \cdot d}$ achieves the following distortion for any worst-case vectors $\x, \y \in \RR^d$ with $\norm{\x}=1$:
\begin{itemize}
    \item $\E\left[ \left< \y, Q_{\tt prod}^{-1}\left( Q_{\tt prod}(\x) \right) \right> \right] = \langle \y, \x \rangle$
    \item $D_{\tt prod}(Q_{\tt prod}) := \E\left[\left| \langle \y, \x \rangle - \langle \y, Q_{\tt prod}^{-1}\left( Q_{\tt prod}(\x) \right) \rangle \right|^2 \right] \le \frac{\sqrt{3} \pi^2 \cdot \norm{\y}_2^2}{d} \cdot \frac{1}{4^{b}}$ for any $b \ge 0$.
    \item For small bit-widths the above distortion upper bound can be further refined. 
    Specifically, for $b = 1, 2, 3, 4$ we have $D_{\tt prod}(Q_{\tt prod}) \approx \frac{\bf 1.57}{d}, \frac{\bf 0.56}{d}, \frac{\bf 0.18}{d}, \frac{\bf 0.047}{d}$, respectively.
\end{itemize}

\paragraph{Lower Bound.}
In \cref{thrm_lower_bound}, we leverage Shannon's lower bound and Yao's minimax principle to prove that for any randomized quantization algorithm $Q: \RR^d \to \{ 0, 1 \}^{b \cdot d}$ with bit-width $b$, there exist hard input instances $\x, \y \in \RR^d$ with $\norm{\x} = 1$ such that the following lower bounds hold:
\begin{itemize}
    \item $D_{\tt mse}(Q) := \E\left[\left\| \x - Q^{-1}\left( Q(\x) \right) \right\|_2^2 \right] \ge \frac{1}{4^{b}}$
    \item $D_{\tt prod}(Q) = \E\left[\left| \langle \y, \x \rangle - \langle \y, Q^{-1}\left( Q(\x) \right) \rangle \right|^2 \right] \ge \frac{\norm{\y}_2^2}{d} \cdot \frac{1}{4^{b}}$
\end{itemize}

As demonstrated by our lower bounds, \TurboQuant's MSE distortion is provably within a factor of at most $\frac{\sqrt{3} \pi}{2} \approx {\bf 2.7}$ of the information-theoretical lower bound.
Notably, for smaller bit-widths, this factor significantly decreases. 
For instance, at a bit-width of $b=1$ \TurboQuant achieves a distortion that is only a factor of approximately ${\bf 1.45}$ away from the optimal which is also confirmed by our experimental results, indicating its efficiency in low-bit-width scenarios.

\paragraph{Experimental Results.}
In \cref{sec:exp_valivation}, we empirically validate our theoretical distortion bounds, demonstrating that \TurboQuant's observed distortions closely align with our predictions across various real-world datasets, approaching the established lower bounds.

Furthermore, in \cref{sec:niah} and \cref{sec:exp_kv_end2end}, we showcase \TurboQuant's efficacy in online KV cache quantization. 
Specifically, we achieve perfect long-context retrieval in needle-in-a-haystack tasks and maintain high performance on other long-context downstream tasks, all while compressing the KV cache by a factor exceeding $5 \times$.

Finally in \cref{sec:nn_exp} we apply \TurboQuant to various high-dimensional near neighbor search tasks.
\TurboQuant consistently outperforms data-dependent product quantization (PQ), while reducing the indexing time to essentially zero.

\section{Preliminaries} \label{sec:prelim}
We use boldface lowercase letters, such as $\x$ and $\y$, to denote vectors, and boldface uppercase letters, like $\M$, to denote matrices.
To denote a slice of a vector $\x$ between the coordinate indices $i$ and $j$ inclusive of the endpoints, we use the notation $\x_{i:j}$. For a matrix $\M$, we write $\M_{i,:}$ to denote its $i$-th row vector, which we will simply refer to as $\M_i$.

We use the notation $\SS^{d-1}$ to denote the hypersphere in $\RR^d$ of radius $1$.
For a random variable $x$ we denote its differential entropy as $h(x)$. 
For random variables $x$ and $y$, the mutual information between them is denoted as $I(x; y) = h(x) - h(x|y)$.

Given that \TurboQuant employs random rotation to mitigate worst-case input scenarios, understanding the statistical properties of random points on a hypersphere is essential. 
The following lemma outlines one such property that we will need for analysis and design purposes:

\begin{lemma}[coordinate distribution of random point on hypersphere]\label{lem_coordinate_distribution}
    For any positive integer $d$ if $\x \in \SS^{d-1}$ is a random variable uniformly distributed over the unit hypersphere, then for any $j \in [d]$ the coordinate $\x_j$ follows the following (scaled/shifted) Beta distribution:
    \[
    \x_j \sim f_{X}(x) := \frac{\Gamma(d/2)}{\sqrt{\pi} \cdot \Gamma((d-1)/2)} \left( 1 - x^2 \right)^{(d-3)/2}.
    \]
    In high dimensions this beta distribtion converges to the normal distribution $f_{X}(\cdot) \to \Ncal(0, 1/d)$.
\end{lemma}
\begin{proof}
$f_X(x)$ equals the ratio of the area of a sphere with radius $\sqrt{1-x^2}$ in dimension $d-1$ to the volume of a unit sphere in dimension $d$ scaled down by $1/\sqrt{1-x^2}$ (by Pythagorean theorem). 
Therefore,
\[
f_X(x) = \frac{\frac{2 \pi^{(d-1)/2}}{\Gamma((d-1)/2)} \cdot (1-x^2)^{(d-2)/2}}{\frac{2 \pi^{d/2}}{\Gamma(d/2)}} \cdot 1/\sqrt{1-x^2}= \frac{\Gamma(d/2)}{\sqrt{\pi} \cdot \Gamma((d-1)/2)} \left( 1 - x^2 \right)^{(d-3)/2}.
\]
\end{proof}

\subsection{Shannon Lower Bound on Distortion}
\label{sec:slb}
The Shannon Lower Bound (SLB) is a powerful tool, derived from Shannon's lossy source coding theorem~\cite{shannon1959coding}, that provides a universal lower bound on the optimal achievable distortion rate for any lossy compression scheme. 
Specifically, we use a version of SLB tailored for the mean-squared error (MSE) distortion measure applied to general $d$-dimensional sources.
\begin{lemma}[SLB]\label{lem_slb}
    Let $\x \in \RR^d$ be a random vector with an arbitrary probability distribution $p_X$ and finite differential entropy $h(\x)$.
    Define the MSE distortion-rate function $D(B)$ for total bit complexity $B \ge 0$ as:
    \[
    D(p_X, B) := \inf \left\{ \E \left[ \norm{\x - \y}_2^2 \right] : I(\x; \y) \le B \right\},
    \]
    where the infimum is taken over all joint distributions of $\x$ and a reconstruction random vector $\y \in \RR^d$ such that the mutual information $I(\x; \y)$ is at most $B$ and $\E \left[ \norm{\x - \y}_2^2 \right]$ is the expected MSE distortion, calculated with respect to the joint distribution of $\x$ and $\y$.
    Then, for any bit complexity $B \ge 0$, the following Shannon Lower Bound holds:
    \[
    D(p_X, B) \ge \frac{d}{2 \pi e} \cdot 2^{(2/d) (h(\x) - B)}.
    \]
\end{lemma}
This is a classic result proved using backward Gaussian test channel (for a proof see \cite{cover1999elements}).
Our lower bound result uses a corollary of SLB that corresponds to the uniformly distributed random points on the unit hyeprsphere.
We present this in the following lemma:
\begin{lemma}[SLB for random point on hypersphere]\label{lem_slb_random_sphere}
Let $\x \in \SS^{d-1}$ be a random variable uniformly distributed over the unit hypersphere and define the MSE distortion-rate function $D(B)$ for total bit complexity $B$ as per \cref{lem_slb}.
Then, for any bit complexity $B \ge 0$, the following distortion lower bound holds:
\[
D(B) \ge 2^{-2B/d}.
\]
\end{lemma}
\begin{proof}
If we let $A_d$ denote the area of the hypersphere $\SS^{d-1}$, the entropy of uniform distribution over hypersphere is $h(\x) = \log_2 A_d$. 
Plugging this into the SLB from \cref{lem_slb} we get $D(B) \ge \frac{d}{2 \pi e} \cdot {A_d}^{2/d} \cdot 2^{-2B/d}$. 
Using Stirling's approximation formula for Gamma function we have $A_d = \frac{2 \pi^{d/2}}{\Gamma(d/2)} \ge \left(\frac{2 \pi e}{d} \right)^{d/2} \cdot \sqrt{\frac{2d}{\pi}} \cdot (1 - O(1/d))$. 
By substituting this into the inequality obtained from \cref{lem_slb} we get the desired lower bound.
\end{proof}


\subsection{QJL: 1-bit inner product quantization}
As previously stated, we design two VQ algorithms: one optimized for minimizing MSE and the other for minimizing inner product error. 
We show that MSE-optimal quantizers do not necessarily provide unbiased inner product estimates, particularly exhibiting significant bias at lower bit-widths.
Our solution for inner product quantization is a two-stage algorithm. 
First, we apply the MSE-optimal quantizer using one less bit than the desired bit-width budget, thus minimizing the L2 norm of the residuals.
Next we apply an unbiased and optimal single-bit quantizer to the residual.
For the single-bit inner product quantizer, we utilize the recently proposed Quantized Johnson-Lindenstrauss (QJL) algorithm~\cite{qjl}, which is an optimal inner product quantizer with a bit-width of one. Here, we present the QJL algorithm and its essential theoretical guarantees.
\begin{definition}[QJL]\label{def_qjl}
For any positive integer $d$ the QJL map $Q_{\tt qjl}: \RR^d \to \{ -1, +1 \}^d$ is defined as:
\[
Q_{\tt qjl}(\x) := {\tt sign} \left( \S \cdot \x \right) ~~~~\text{ for any } \x \in \RR^d,
\]
where $\S \in \RR^{d \times d}$ is a random matrix with i.i.d. entries sampled from the normal distribution $\Ncal(0, 1)$ and the ${\tt sign}$ function is applied entry-wise to its vector input.
The inverse/dequantization map $Q_{\tt qjl}^{-1}: \{ -1, +1\}^d \to \RR^d$ is defined as:
\[
Q_{\tt qjl}^{-1}(\z) := \frac{\sqrt{\pi/2}}{d} \cdot \S^\top \cdot \z ~~~~\text{ for any } \z \in \{ -1, +1\}^d.
\]
\end{definition}

In the next lemma we restate the results from~\cite{qjl} that show the QJL is unbiased and also has small inner product distortion:
\begin{lemma}[performance guarantee: QJL]\label{lem_qjl}
Let $Q_{\tt qjl}$ and $Q_{\tt qjl}^{-1}$ be defined as per \cref{def_qjl}. 
For any vector $\x \in \SS^{d-1}$ and any $\y \in \RR^d$ we have the following:
\begin{itemize}
    \item Unbiased: $\E\left[ \left< \y, Q_{\tt qjl}^{-1}\left( Q_{\tt qjl}(\x) \right) \right> \right] = \langle \y, \x \rangle$.
    \item Variance Bound: $\mathtt{Var} \left( \left< \y, Q_{\tt qjl}^{-1}\left( Q_{\tt qjl}(\x) \right) \right> \right) \le \frac{\pi}{2 d} \cdot \norm{\y}_2^2$
\end{itemize}
\end{lemma}
\begin{proof}
The unbiasedness immediately follows from Lemma~3.2 of \cite{qjl}.
To show the variance bound let $\s_1, \s_2, \ldots \s_m$ denote the rows of the random matrix $\S$ in \cref{def_qjl}. We have:
    \[
    \left< \y, Q_{\tt qjl}^{-1}\left( Q_{\tt qjl}(\x) \right) \right> = \frac{1}{d} \sum_{i\in[d]} \sqrt{\pi/2} \cdot \s_i^\top \y \cdot {\tt sign}(\s_i^\top \x).
    \]
    Since $\s_i$'s are i.i.d. the above is indeed the average of $d$ i.i.d. random samples defined as $z_i := \sqrt{\pi/2} \cdot \s_i^\top \y \cdot {\tt sign}(\s_i^\top \x)$ for $i \in [d]$.
    Let us now upper bound the variance of a single $z_i$ using Fact~3.4 from \cite{qjl}:
    \begin{equation}\label{eq_moments_estimator}
    \mathtt{Var} \left( z_i \right) = \pi/2 \cdot \mathtt{Var} \left( \s_i^\top \y \cdot {\tt sign}(\s_i^\top \x) \right) \le \pi/2 \cdot \E \left[ (\s_i^\top \y)^2 \right] = \pi/2 \cdot \norm{y}_2^2,
    \end{equation}
    where the last equality above follows because $\s_i^\top \y$ is a Gaussian random variable with mean zero and variance $\|\y\|_2^2$.
    Now the variance of the average of $d$ i.i.d. random samples $z_1, z_2, \ldots z_d$ is:
    \[
    \mathtt{Var} \left( \left< \y, Q_{\tt qjl}^{-1}\left( Q_{\tt qjl}(\x) \right) \right> \right) = \frac{1}{d^2} \sum_{i\in[d]} \mathtt{Var} ( z_i ) \le \frac{\pi}{2 d} \cdot \norm{\y}_2^2.
    \]
\end{proof}

\section{\TurboQuant: High Performance Quantization}
We developed two VQ algorithms, each tailored to a specific objective. 
The first algorithm is designed to minimize the MSE between the original and reconstructed vectors after quantization. 
The second algorithm is optimized for unbiased inner product estimation, addressing the bias inherent in MSE-optimal quantizers. 
These algorithms are detailed in the following subsections.

Furthermore, in \cref{sec:lower_bound}, we establish information-theoretic lower bounds on the best achievable distortion rates for any vector quantizer. 
This analysis demonstrates that \TurboQuant achieve near-optimality, differing from the lower bound by only a small constant factor across all bit-widths.

\subsection{MSE Optimal \TurboQuant} \label{sec:mse_turbo_alg}
Let $\x \in \SS^{d-1}$ be a (worst-case) vector on the unit sphere in dimension $d$.
We aim to quantize $\x$ to $b$ bits per coordinate while minimizing the reconstruction MSE defined in \cref{eq:mse}.
We start by randomizing this vector by multiplying it with a random rotation matrix $\BPi \in \RR^{d \times d}$. 
We can generate $\BPi$ by applying QR decomposition on a random matrix with i.i.d Normal entries.

The resulting rotated vector, $\BPi \cdot \x$, is uniformly distributed on the unit sphere $\SS^{d-1}$.
As shown in \cref{lem_coordinate_distribution}, each coordinate of $\BPi \cdot \x$ follows a Beta distribution, which converges to a normal distribution in high dimensions.
Furthermore, in high dimensions, distinct coordinates of $\BPi \cdot \x$ become nearly independent~\cite{vershynin2018high}, allowing us to apply optimal scalar quantizers to each coordinate independently.
Therefore, by \cref{lem_coordinate_distribution}, our task reduces to designing a scalar quantizer for random variables with the distribution $f_{X}(x) = \frac{\Gamma(d/2)}{\sqrt{\pi} \cdot \Gamma((d-1)/2)} \left( 1 - x^2 \right)^{(d-3)/2}$ for $x \in [-1, 1]$.

The optimal scalar quantization problem, given a known probability distribution, can be framed as a continuous k-means problem in dimension one. 
Specifically, we aim to partition the interval $[-1, 1]$ into $2^b$ clusters/buckets.
The optimal solution adheres to a Voronoi tessellation~\cite{lloyd1982least}, meaning interval boundaries are the midpoints between consecutive centroids, when arranged in sorted order.
Therefore, with $c_i$'s denoting the centroids in ascending order, we can formulate the scalar quantization as the following k-means optimization problem:
\begin{equation}\label{eq:continuous_k_means}
    \Ccal(f_X, b) := \min_{-1 \le c_1 \le c_2 \le \ldots \le c_{2^b} \le 1} \sum_{i=1}^{2^b} \int_{\frac{c_{i-1} + c_i}{2}}^{\frac{c_i + c_{i+1}}{2}} |x - c_i|^2 \cdot f_X(x)~dx.
\end{equation}
Note that $\Ccal(f_X, b)$ in \cref{eq:continuous_k_means} denotes the optimal MSE cost function for bit-width $b$, a quantity we will bound to prove the upper bound on the end-to-end MSE of \TurboQuant.
The problem in \cref{eq:continuous_k_means} can be solved using iterative numerical methods to achieve any desired precision.
We solve \cref{eq:continuous_k_means} for a range of practically relevant bit-widths $b$ once, and store the results for future uses by the quantizer.

For example, in moderately high dimensions $d$, where the distribution $f_X(x)$ closely approximates a normal distribution, the optimal quantization centroids for bit-widths $b = 1, 2$ are $\left\{ \pm \frac{\sqrt{2/\pi}}{\sqrt{d}} \right\}$ and $\left\{ \pm \frac{0.453}{\sqrt{d}}, \pm \frac{1.51}{\sqrt{d}} \right\}$, respectively.

Therefore the quantizer $Q_{\tt mse}: \RR^d \to \{ 0, 1 \}^{b \cdot d}$ first computes $\BPi \cdot \x$ and then computes and stores the indices of the nearest centroids to each coordinate of this vector.
The dequantization map $Q_{\tt mse}^{-1}: \{ 0, 1 \}^{b \cdot d} \to \RR^d$ reconstructs the vector by retrieving the centroids corresponding to the stored indices and then rotating the result back to the original basis through multiplication with $\BPi^\top$.
A pseudocode for these procedures is given in \cref{alg_turboquant_mse}.

\begin{algorithm}[t!]
\caption{$\TurboQuant_{\tt mse}$: optimized for MSE}\label{alg_turboquant_mse}
\begin{algorithmic}[1]
    \STATE{\bfseries input:} dimension $d$ and bit-width $b$

    {\ttfamily\textcolor{blue}{// Global Parameters for Setting up $\TurboQuant_{\tt mse}$}} \\
    \STATE Generate a {\ttfamily\textcolor{blue}{random rotation matrix}} $\BPi \in \RR^{d \times d}$

    \STATE Construct {\ttfamily\textcolor{blue}{codebook}} by finding centroids $c_1, c_2, \ldots c_{2^b} \in [-1, 1]$ that minimize MSE cost in \cref{eq:continuous_k_means} \label{alg:codebook_construction}
   \\\hrulefill
   \STATE{\bf Procedure} $\textsc{Quant}_{\tt mse} ( \x )$
   \STATE $\y \gets \BPi \cdot \x$\label{line_y}
   \STATE ${\tt idx}_j \gets \arg\min_{k \in [2^b]}  \abs{\y_j - c_k}$ for every $j \in [d]$\label{line_idx} \hfill \COMMENT{{\ttfamily\textcolor{blue}{ ${\tt idx}_j$'s are $b$-bit integers}}}
    \STATE {\bf output: }${\tt idx}$
   \\\hrulefill
   \STATE{\bf Procedure} $\textsc{DeQuant}_{\tt mse} ({\tt idx})$
   \STATE $\tilde{\y}_j \gets c_{{\tt idx}_j}$ for every $j \in [d]$\label{y_tilde}
   \STATE $\tilde{\x} \gets \BPi^\top \cdot \tilde{\y}$
   \STATE {\bf output: }$\tilde{\x}$ 
\end{algorithmic}
\end{algorithm}

We are now ready to prove our main theorem for $\TurboQuant_{\tt mse}$.
\begin{theorem}[performance guarantee: $\TurboQuant_{\tt mse}$]\label{thrm_mse}
For any bit-width $b \ge 1$ and any vector $\x \in \SS^{d-1}$, the procedure $\textsc{Quant}_{\tt mse}(\x)$ in \cref{alg_turboquant_mse} outputs an index vector ${\tt idx} \in [2^b]^d$. 
When this index vector is passed to the primitive $\textsc{DeQuant}_{\tt mse}({\tt idx})$, it produces a reconstructed vector $\tilde{\x} \in \RR^d$ that satisfies the following distortion bounds:
\begin{itemize}
    \item MSE defined as $D_{\tt mse} := \E_{\tilde{\x}}[ \norm{\x - \tilde{\x}}_2^2 ]$ is bounded by $D_{\tt mse} \le \frac{\sqrt{3} \pi}{2} \cdot \frac{1}{4^{b}}$ for any $b \ge 0$.
    \item For small bit-widths, specifically $b = 1, 2, 3, 4$ the MSE exhibits finer-grained distortion values: $D_{\tt mse} \approx {\bf 0.36}, {\bf 0.117}, {\bf 0.03}, {\bf 0.009}$, respectively.
\end{itemize}
\end{theorem}
\begin{proof}
We start the proof by showing that $D_{\tt mse} = d \cdot \Ccal(f_X, b)$, where $\Ccal(f_X, b)$ is the optimal MSE cost for scalar quantizer defined in \cref{eq:continuous_k_means}.
Let $\tilde{\y}$ be defined as per line~\ref{y_tilde} of \cref{alg_turboquant_mse}.
Since $\BPi$ is a rotation matrix we can write:
$\norm{\x - \tilde{\x}}_2 = \norm{\BPi \cdot \x - \tilde{\y}}_2$.
Using the notation $\y = \BPi \cdot \x$ as per line~\ref{line_y} of \cref{alg_turboquant_mse} and plugging this into the definition of $D_{\tt mse}$ we can write:
\begin{align*}
    D_{\tt mse} &= \E [\norm{\y - \tilde{\y}}_2^2] \\
    &= \sum_{j \in [d]} \E\left[ |\y_j - \tilde{\y}_j|^2 \right] \\
    &= \sum_{j \in [d]} \E\left[ |\y_j - c_{{\tt idx}_j}|^2 \right]\\
    &= d \cdot \E\left[ |\y_1 - c_{{\tt idx}_1}|^2 \right] \\
    &= d \cdot \min_{-1 \le c_1 \le c_2 \le \ldots \le c_{2^b} \le 1} \sum_{i=1}^{2^b} \int_{\frac{c_{i-1} + c_i}{2}}^{\frac{c_i + c_{i+1}}{2}} |x - c_i|^2 \cdot f_X(x)~dx \\
    &= d \cdot \Ccal(f_X, b).
\end{align*}
The third equality above follows from the definition of $\tilde{\y}$ in line~\ref{y_tilde} of \cref{alg_turboquant_mse} and the fourth line above follows because all $\y_j$'s have identical distribution of $\y_j \sim f_X(\cdot)$ as shown in \cref{lem_coordinate_distribution}.
The last two lines above follows because $c_{{\tt idx}_j}$ is chosen to be the nearest centroid to each coordinate $\y_j$ in line~\ref{line_idx}.

Now we must bound the optimal k-means cost $\Ccal(f_X, b)$. 
For moderate values of $d$, $f_X \to \Ncal(0, 1/d)$. 
By numerically solving the optimization problem in \cref{eq:continuous_k_means} for values $b = 1, 2, 3, 4$ we get that $\Ccal(f_X, b) \approx \frac{0.36}{d}, \frac{0.117}{d}, \frac{0.03}{d}, \frac{0.009}{d}$, respectively.
For larger bit-widths $b > 4$, we can apply the Panter-Dite~\cite{panter1951quantization} high-resolution formula for the distortion of a fixed-rate scalar quantizer, yielding the following bound:
\[
\Ccal(f_X, b) \le \frac{1}{12} \cdot \left( \int f_X(x)^{1/3}~dx \right)^3 \cdot \frac{1}{4^b} =  \frac{\sqrt{3} \pi }{2 d} \cdot \frac{1}{4^b}.
\]
This completes the proof.
\end{proof}

\paragraph{Entropy Encoding Codebook Pointers.}
\TurboQuant's efficiency can be further increased by applying entropy encoding to the indices that point to the closest codebook elements. 
Specifically, the probability of each codeword index appearing in the quantized vectors can be computed as $p_\ell := \int_{\frac{c_{\ell-1} + c_\ell}{2}}^{\frac{c_\ell + c_{\ell+1}}{2}} f_X(x)~dx$.
Optimally coding the indices, reduces the average bit-width to nearly the entropy of the distribution $\{ p_i \}_{i \in [2^b]}$.
This lossless compression does not affect the distortion and provides a bit-width reduction at no cost.
The most significant reduction occurs for $b=4$, where the entropy of $\{ p_i \}_{i \in [2^b]}$ is approximately $3.8$.
Detailed calculations for optimal prefix codes reveal that the average bit-width can be reduced by $5 \%$.
However, given the limited gain, we have chosen not to incorporate this technique into \TurboQuant to maintain simplicity and speed.

\subsection{Inner-product Optimal \TurboQuant} \label{sec:prod_turbo_alg}
For important applications like nearest neighbor search, having an unbiased inner product estimator is essential.
However, $\TurboQuant_{\tt mse}$ presented in \cref{sec:mse_turbo_alg} does not provide unbiased inner product estimates with query vectors.
To illustrate this, consider the case with a bit-width of $b=1$.
In this scenario, the optimal codebooks that solve the optimization problem in \cref{eq:continuous_k_means}, for sufficiently large $d$, are $\left\{ \pm \sqrt{\frac{2}{\pi d}} \right\}$. 
This implies that the quantization map for $\TurboQuant_{\tt mse}$ is $Q_{\tt mse}(\x) = {\tt sign} \left( \BPi \cdot \x \right)$ for any $\x \in \RR^d$, and the dequantization map is $Q_{\tt mse}^{-1}(\z) = \sqrt{\frac{2}{\pi d}} \cdot \BPi^\top \cdot \z$ for any $\z \in \{ -1, +1\}^d$. 
Therefore, for large enough $d$, according to \cref{lem_qjl}, we have $\E\left[ \left< \y, Q_{\tt mse}^{-1}\left( Q_{\tt mse}(\x) \right) \right> \right] = \frac{2}{\pi} \cdot \langle \y, \x \rangle$, which has a multiplicative bias of $2/\pi$.
This bias diminishes with increasing bit-widths $b$, as we empirically demonstrate in \cref{sec:exp_valivation}.

To address this bias, we propose a solution that combines $\TurboQuant_{\tt mse}$ with an instance of QJL~\cite{qjl}.
Specifically, let $Q_{\tt mse}$ be the quantization map corresponding to $\TurboQuant_{\tt mse}$ with a bit-width of $b-1$. 
For any $\x \in \SS^{d-1}$ the residual vector, defined as $\r := \x - Q_{\tt mse}^{-1}\left( Q_{\tt mse}(\x) \right)$, has a small L2 norm, i.e., on expectation $\E[\norm{\r}] = \sqrt{\Ccal(f_X, b-1)}$ (per \cref{eq:continuous_k_means}).
We can then apply the QJL quantization map $Q_{\tt qjl}$ on this residual vector, resulting in an overall bit-width of $b$ and providing the following unbiased inner product estimator:
\[
\left< \y, Q_{\tt mse}^{-1}\left( Q_{\tt mse}(\x) \right) \right> + \norm{\r}_2 \cdot \left< \y, Q_{\tt qjl}^{-1}\left( Q_{\tt qjl}(\r) \right) \right>.
\]
More formally, the quantization map $Q_{\tt prod}: \SS^{d-1} \to [2^{b-1}]^d \times \{ -1, 1 \}^d \times \RR$ is defined as:
\[
Q_{\tt prod}(\x) = \left[ Q_{\tt mse}(\x), Q_{\tt qjl}\left( \x - Q_{\tt mse}^{-1}\left( Q_{\tt mse}(\x) \right) \right), \norm{\x - Q_{\tt mse}^{-1}\left( Q_{\tt mse}(\x) \right)}_2 \right].
\]
A pseudocode for this procedure is given in \cref{alg_turboquant_prod}.

\begin{algorithm}[t!]
\caption{$\TurboQuant_{\tt prod}$: optimized for inner product}\label{alg_turboquant_prod}
\begin{algorithmic}[1]
    \STATE{\bfseries input:} dimension $d$ and bit-width $b$

    {\ttfamily\textcolor{blue}{// Global Parameters for Setting up $\TurboQuant_{\tt prod}$}} \\
    \STATE Instantiate a {\ttfamily\textcolor{blue}{$\TurboQuant_{\tt mse}$}} with bit-width $b-1$ as per \cref{alg_turboquant_mse}\label{turbo_mse}
    \STATE Generate a {\ttfamily\textcolor{blue}{random projection matrix}} $\S \in \RR^{d \times d}$ with i.i.d. entries $\S_{i,j} \sim \Ncal(0, 1)$\label{random_s}
   \\\hrulefill
   \STATE{\bf Procedure} $\textsc{Quant}_{\tt prod} ( \x )$
   \STATE ${\tt idx} \gets \textsc{Quant}_{\tt mse}(\x)$
   \STATE $\r \gets \x - \textsc{DeQuant}_{\tt mse}({\tt idx})$ \label{line_residual} \hfill \COMMENT{{\ttfamily\textcolor{blue}{ residual vector}}} 
   \STATE ${\tt qjl} \gets {\tt sign} \left( \S \cdot \r \right)$ \hfill \COMMENT{{\ttfamily\textcolor{blue}{ QJL on residual vector}}}
   \STATE {\bf output: }$({\tt idx}, {\tt qjl}, \norm{\r}_2)$
   \\\hrulefill
   \STATE{\bf Procedure} $\textsc{DeQuant}_{{\tt prod}} ( {\tt idx}, {\tt qjl}, \gamma )$
   \STATE $\tilde{\x}_{\tt mse} \gets \textsc{DeQuant}_{\tt mse}({\tt idx})$ \label{line_tilde_x}
   \STATE $\tilde{\x}_{\tt qjl} \gets \frac{\sqrt{\pi/2}}{d} \cdot \gamma \cdot \S^\top \cdot {\tt qjl}$ \label{line_tilde_x_qjl}
   \STATE {\bf output: }$\tilde{\x}_{\tt mse} + \tilde{\x}_{\tt qjl}$ \label{output_tilde_x}
\end{algorithmic}
\end{algorithm}

We prove the main result for $\TurboQuant_{\tt prod}$ in the following theorem.
\begin{theorem}[performance guarantee: $\TurboQuant_{\tt prod}$] \label{thrm_prod}
For any bit-width $b \ge 1$ and any vector $\x \in \SS^{d-1}$, the procedure $\textsc{Quant}_{\tt prod}(\x)$ in \cref{alg_turboquant_prod} outputs an index vector ${\tt idx} \in [2^{b-1}]^d$ along with a sign vector ${\tt qjl} \in \{ -1, 1 \}^d$ and a positive number $\gamma \ge 0$. 
When these vectors and the scalar value are passed to the primitive $\textsc{DeQuant}_{\tt prod}({\tt idx}, {\tt qjl}, \gamma)$, it produces a reconstructed vector $\tilde{\x} \in \RR^d$ that for any vector $\y \in \RR^d$ satisfies the following properties:
\begin{itemize}
    \item Expected inner-product $\E_{\tilde{\x}}\left[ \left< \y, \tilde{\x} \right> \right] = \langle \y, \x \rangle$
    \item Inner-product distortion defined as $D_{\tt prod} := \E_{\tilde{\x}}\left[\left| \langle \y, \x \rangle - \langle \y, \tilde{\x} \rangle \right|^2 \right]$ is bounded by $D_{\tt prod} \le \frac{\sqrt{3} \pi^2 \cdot \norm{\y}_2^2}{d} \cdot \frac{1}{4^{b}}$ for any $b \ge 0$.
    \item For small bit-widths, specifically $b = 1, 2, 3, 4$, $D_{\tt prod}$ exhibits finer-grained distortion values: $D_{\tt prod} \approx \frac{\bf 1.57}{d}, \frac{\bf 0.56}{d}, \frac{\bf 0.18}{d}, \frac{\bf 0.047}{d}$, respectively.
\end{itemize}
\end{theorem}
\begin{proof}
First we compute the conditional expectation of the inner product estimate $\langle \y, \tilde{\x} \rangle$ conditioned on $\tilde{\x}_{\tt mse}$ as follows:
\begin{align*}
    \E \left[ \langle \y, \tilde{\x} \rangle | \tilde{\x}_{\tt mse} \right] &= \E_{\tilde{\x}_{\tt qjl}} \left[ \langle \y, \tilde{\x}_{\tt mse} + \tilde{\x}_{\tt qjl} \rangle | \tilde{\x}_{\tt mse} \right]\\
    &= \langle \y, \tilde{\x}_{\tt mse} \rangle + \E_{\tilde{\x}_{\tt qjl}} \left[ \langle \y, \tilde{\x}_{\tt qjl} \rangle | \tilde{\x}_{\tt mse} \right] \\
    &= \langle \y, \tilde{\x}_{\tt mse} \rangle + \langle \y , \r \rangle\\
    &= \langle \y , \x \rangle,
\end{align*}
where the first equality follows from the definition of $\tilde{\x}$ in line~\ref{output_tilde_x} of the algorithm.
The third equality above follows from \cref{lem_qjl} and last line follows from definition of the residual vector $\r = \x - \tilde{\x}_{\tt mse}$ in line~\ref{line_residual}.
Now we can computed the unconditional expectation using the law of total expectation:
$\E_{\tilde{\x}}\left[ \left< \y, \tilde{\x} \right> \right] = \E_{\tilde{\x}_{\tt mse}} \left[\E \left[ \langle \y, \tilde{\x} \rangle | \tilde{\x}_{\tt mse} \right] \right] = \E[\langle \y , \x \rangle] = \langle \y , \x \rangle$, which proves the first claim of the theorem.

We apply the same conditioning on $\tilde{\x}_{\tt mse}$, when computing the distortion, and then compute the resulting conditional distortion:
\begin{align*}
    \E\left[\left. \left| \langle \y, \x \rangle - \langle \y, \tilde{\x} \rangle \right|^2 \right| \tilde{\x}_{\tt mse} \right] &= \E_{\tilde{\x}_{\tt qjl}}\left[\left. \left| \langle \y, \x \rangle - \langle \y, \tilde{\x}_{\tt mse} + \tilde{\x}_{\tt qjl} \rangle \right|^2 \right| \tilde{\x}_{\tt mse} \right] \\
    &= \E_{\tilde{\x}_{\tt qjl}}\left[\left. \left| \langle \y, \r \rangle - \langle \y, \tilde{\x}_{\tt qjl} \rangle \right|^2 \right| \tilde{\x}_{\tt mse} \right] \\
    &= \mathtt{Var} \left( \left. \langle \y, \tilde{\x}_{\tt qjl} \rangle \right| \tilde{\x}_{\tt mse} \right) \\
    &\le \frac{\pi}{2d} \cdot \norm{\r}_2^2 \norm{\y}_2^2,
\end{align*}
where the second equality above follows from the definitions of $\r$ and $\tilde{\x}_{\tt mse}$ in lines~\ref{line_residual} and \ref{line_tilde_x} of \cref{alg_turboquant_prod}.
The third line above follows because $\E[\langle \y, \tilde{\x}_{\tt qjl} \rangle] = \langle \y, \r \rangle$, by \cref{lem_qjl}. 
The last line follows from the variance bound of QJL estimator shown in \cref{lem_qjl} and using the fact that $\tilde{\x}_{\tt qjl}$ in line~\ref{line_tilde_x_qjl} is re-scaled by $\gamma = \norm{\r}$.

Now by law of total expectation along with the fact that $\r = \x - \tilde{\x}_{\tt mse}$ we can bound the inner product distortion as follows:
\begin{align*}
    D_{\tt prod} &= \E_{\tilde{\x}_{\tt mse}} \left[ \E\left[\left. \left| \langle \y, \x \rangle - \langle \y, \tilde{\x} \rangle \right|^2 \right| \tilde{\x}_{\tt mse} \right] \right]\\
    &\le \frac{\pi}{2 d} \cdot \norm{\y}_2^2 \cdot \E[\norm{\x - \tilde{\x}_{\tt mse}}_2^2]\\
    &= \frac{\pi}{2 d} \cdot \norm{\y}_2^2 \cdot D_{\tt mse}.
\end{align*}
The theorem follows by invoking the MSE bounds from \cref{thrm_mse} with bit-width $b-1$.
\end{proof}

\subsection{Lower Bounds}
\label{sec:lower_bound}

We show that \TurboQuant achieves an optimal distortion rate, up to a small constant factor, for any bit-width by proving lower bounds on the best achievable distortion for any compression algorithm. 
Our lower bound proof leverages Yao's minimax principle. 
This principle allows us to relate the lower bound for randomized algorithms with worst-case deterministic input vectors to the lower bound for deterministic algorithms with randomized input vectors. 
Subsequently, we derive a lower bound on the achievable distortion rate for the latter using Shannon's lower bound (SLB) presented in \cref{sec:slb}.
Formally, we prove the following theorem.

\begin{theorem}[lower bound on best achievable compression distortion]
\label{thrm_lower_bound}
For any randomized quantization algorithm $Q: \SS^{d-1} \to \{ 0, 1 \}^{b \cdot d}$ with bit-width $b$ and any reconstruction map $Q^{-1}: \{ 0, 1 \}^{b \cdot d} \to \RR^d$, there exist a hard input instance $\x \in \SS^{d-1}$ such that:
\[
D_{\tt mse}(Q) := \E\left[\left\| \x - Q^{-1}\left( Q(\x) \right) \right\|_2^2 \right] \ge \frac{1}{4^{b}}.
\]
Furthermore, there exists a $\y \in \SS^{d-1}$ such that:
\[
D_{\tt prod}(Q) = \E\left[\left| \langle \y, \x \rangle - \langle \y, Q^{-1}\left( Q(\x) \right) \rangle \right|^2 \right] \ge \frac{1}{d} \cdot \frac{1}{4^{b}}
\]
\end{theorem}
\begin{proof}
By Yao's minimax principle the expected MSE of the optimal randomized compression algorithm for worst-case inputs ($D_{\tt mse}$) is equal to the expected MSE of the optimal deterministic compression algorithm when applied to inputs drawn from a maximally difficult randomized distribution. 
By definition, the MSE of the latter scenario is lower-bounded by the best achievable MSE for inputs uniformly distributed on the unit hypersphere.

The best achievable MSE for a compression algorithm with bit-width $b$, operating on uniformly distributed inputs from the sphere $\SS^{d-1}$, is lower bounded in \cref{lem_slb_random_sphere}.
Therefore, by invoking \cref{lem_slb_random_sphere} we conclude that $D_{\tt mse} \ge \frac{1}{4^b}$.

Furthermore, from $D_{\tt mse} \ge \frac{1}{4^b}$ and using the definition of $D_{\tt mse}$ we conclude that:
\begin{align*}
    D_{\tt mse} &= \sum_{j=1}^d \E \left[\left| \x_j - \left[ Q^{-1}\left( Q(\x) \right) \right]_j \right|^2 \right]\\
    &= \sum_{j=1}^d \E \left[\left| \langle \e_j, \x \rangle - \langle \e_j, Q^{-1}\left( Q(\x) \right) \rangle \right|^2 \right]\\
    &\ge \frac{1}{4^b}.
\end{align*}
By pigeonhole principle there exist an index $j \in [d]$ such that $\E \left[\left| \langle \e_j, \x \rangle - \langle \e_j, Q^{-1}\left( Q(\x) \right) \rangle \right|^2 \right] \ge \frac{1}{d} \cdot \frac{1}{4^b}$, which completes the proof.
\end{proof}

We note that a comparable lower bound for the \emph{worst-case} distortion in vector quantization can be derived using ``sphere packing'' arguments (indeed, with larger constants as this is a harder problem)~\cite{gersho1982structure}.
However, \cref{thrm_lower_bound} offers a more robust and relevant lower bound for our analysis. 
This is because it establishes a lower bound on the \emph{expected distortion}, rather than the worst-case error, and aligns seamlessly with our upper bounds presented in \cref{thrm_mse} and \cref{thrm_prod}.

\section{Experiments}\label{sec:exp}
All experiments are performed using a single NVIDIA A100 GPU. 
The experimental section is divided into two parts: one to empirically validate the theoretical results, and another to evaluate the performance of our methods on downstream tasks, specifically KV cache quantization and nearest neighbor vector search.

\subsection{Empirical Validation}
\label{sec:exp_valivation}

\begin{figure}[th]
\begin{center}
    \begin{subfigure}{\textwidth}
        \caption{$\TurboQuant_\text{\tt prod}$}
        \centering
        \includegraphics[width=\textwidth]{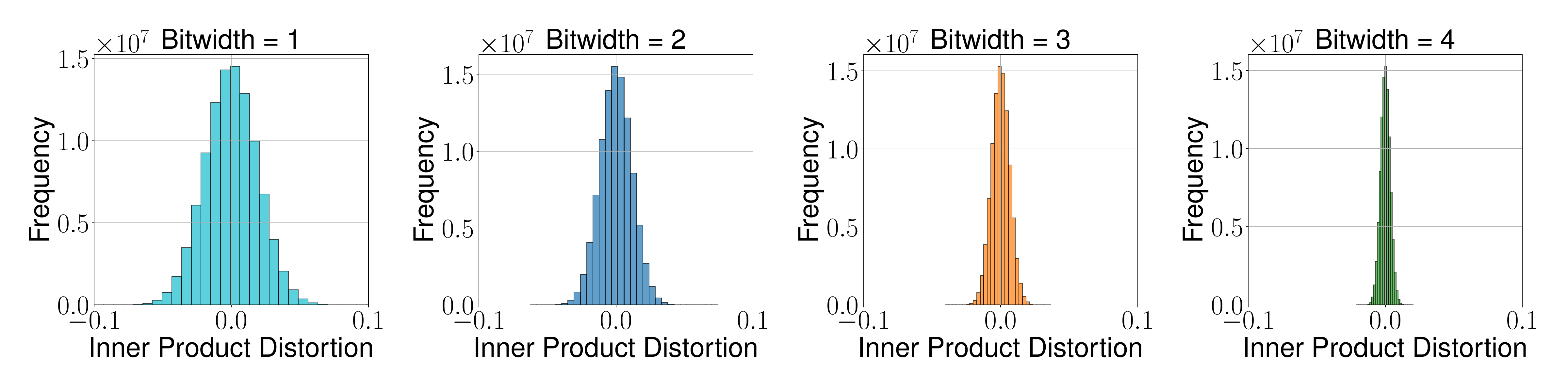}
    \end{subfigure}
    \begin{subfigure}{\textwidth}
        \caption{$\TurboQuant_\text{\tt mse}$}
        \centering
        \includegraphics[width=\textwidth]{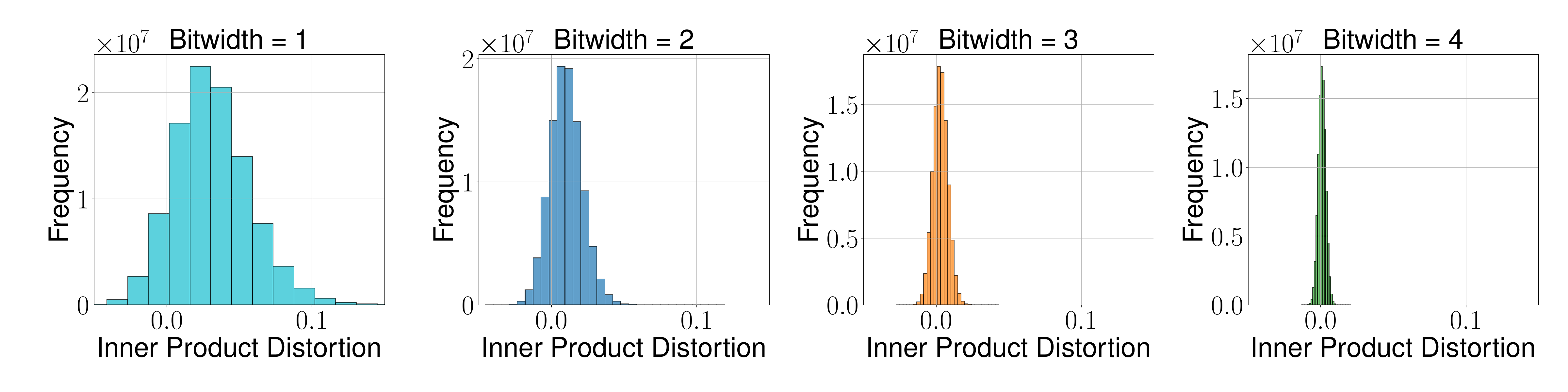}
    \end{subfigure}
    \caption{Error distribution of $\TurboQuant_\text{\tt prod}$ and $\TurboQuant_\text{\tt mse}$ for Inner Product Estimation.}
    \label{fig:inner_distortion}
\end{center}
\end{figure}

In this section, we verify the theoretical results established in previous sections. We conduct our experiments using the DBpedia Entities dataset, which has been encoded into a 1536-dimensional space using OpenAI3 embeddings. 
To perform our experiments, we randomly sample 100,000 data points from the dataset, denoted as training set, which serves as our primary dataset. Additionally, we extract 1,000 distinct entries, denoted as query set, to be used as query points.

We evaluate two quantization methods: \(\TurboQuant_\text{\tt prod}\) and \(\TurboQuant_\text{\tt mse}\). The method \(\TurboQuant_\text{\tt mse}\) is designed to be optimzed for estimating the mean squared error (MSE) between the quantized and original vectors. In contrast, \(\TurboQuant_\text{\tt prod}\) is unbiased for estimating the inner product between the quantized and original vectors.

\begin{figure}[h]
\begin{center}
    \begin{subfigure}{\textwidth}
        \caption{$\TurboQuant_{\tt prod}$}
        \centering
        \includegraphics[width=\textwidth]{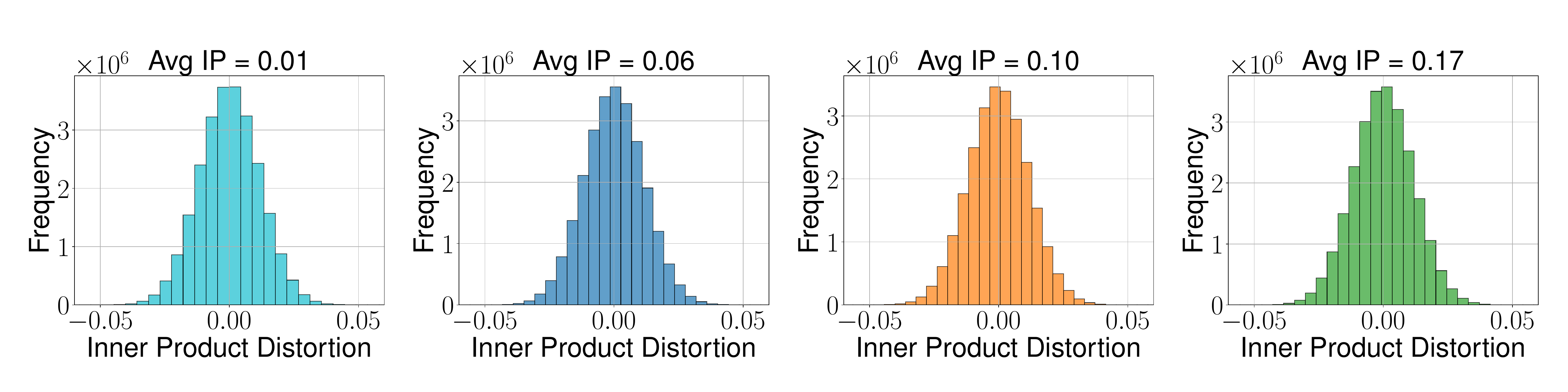}
    \end{subfigure}
    \begin{subfigure}{\textwidth}
        \caption{$\TurboQuant_{\tt mse}$}
        \centering
        \includegraphics[width=\textwidth]{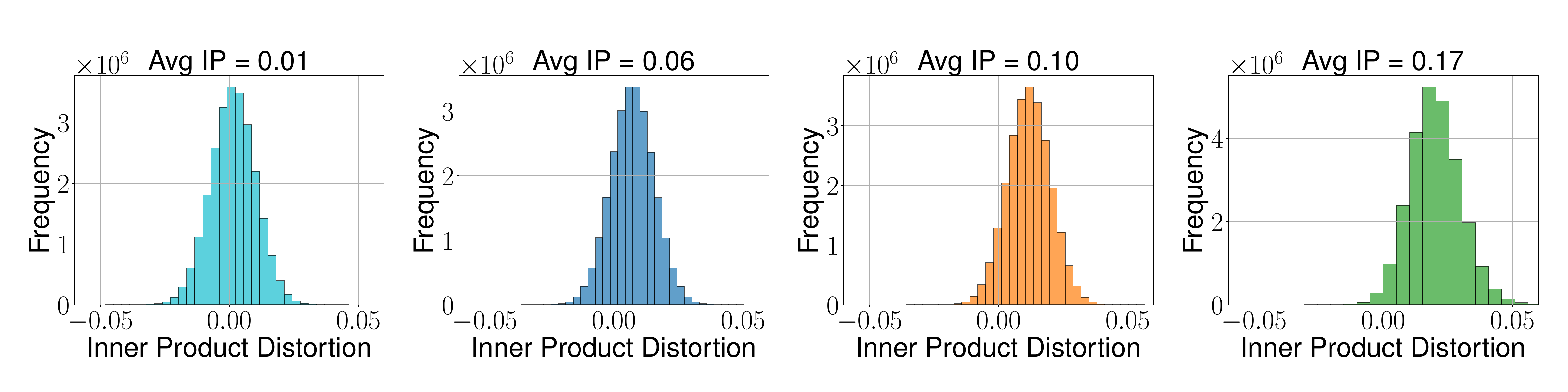}
    \end{subfigure}
    \caption{The variance of Inner-product error remains constant for $\TurboQuant_{\tt prod}$, while in $\TurboQuant_{\tt mse}$ increases with the average inner product. Bit-width is $b=2$.}
    \label{fig:inner_prod_2bit}
\end{center}
\end{figure}

Both methods are applied to the task of inner product estimation by quantizing training set and analyzing the distortion in inner product calculations across different bit widths. As shown in \cref{fig:inner_distortion}, increasing the bit width reduces variance in both methods. However, when used for inner product estimation, \(\TurboQuant_\text{\tt mse}\) introduces bias. This bias diminishes as the bit width increases and eventually converges to zero.

The experimental results, illustrated in \cref{fig:inner_distortion}, confirm that \(\TurboQuant_\text{\tt prod}\) remains unbiased for inner product estimation across all bit widths, while \(\TurboQuant_\text{\tt mse}\) gradually improves with increasing bit width.

As observed in \cref{fig:inner_prod_2bit}, when quantizing to 2 bits, the variance remains constant regardless of the inner product of the original vector in the $\TurboQuant\text{\tt prod}$ approach. However, the same plot indicates that the bias in the $\TurboQuant\text{\tt mse}$ approach is dependent on the average inner product. As the average inner product increases, the bias also increases.

\begin{figure}[h]
    \begin{center}
    \begin{tabular}{cc}
        \begin{subfigure}{0.4\textwidth}
            \caption{\textbf{inner-prod error}}
            \centering
            \includegraphics[width=\textwidth]{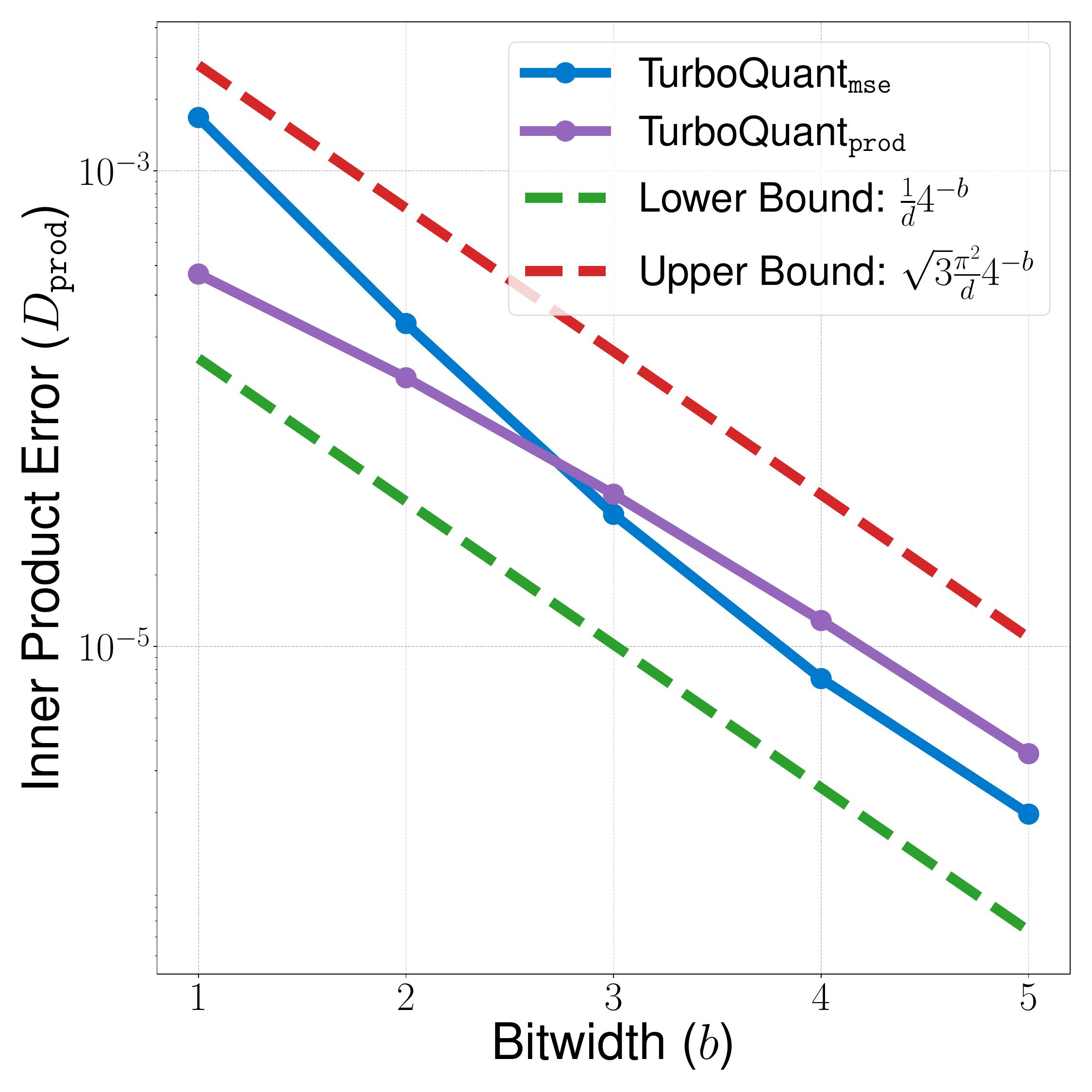}
        \end{subfigure}&
        \begin{subfigure}{0.4\textwidth}
            \caption{\textbf{MSE}}
            \centering
            \includegraphics[width=\textwidth]{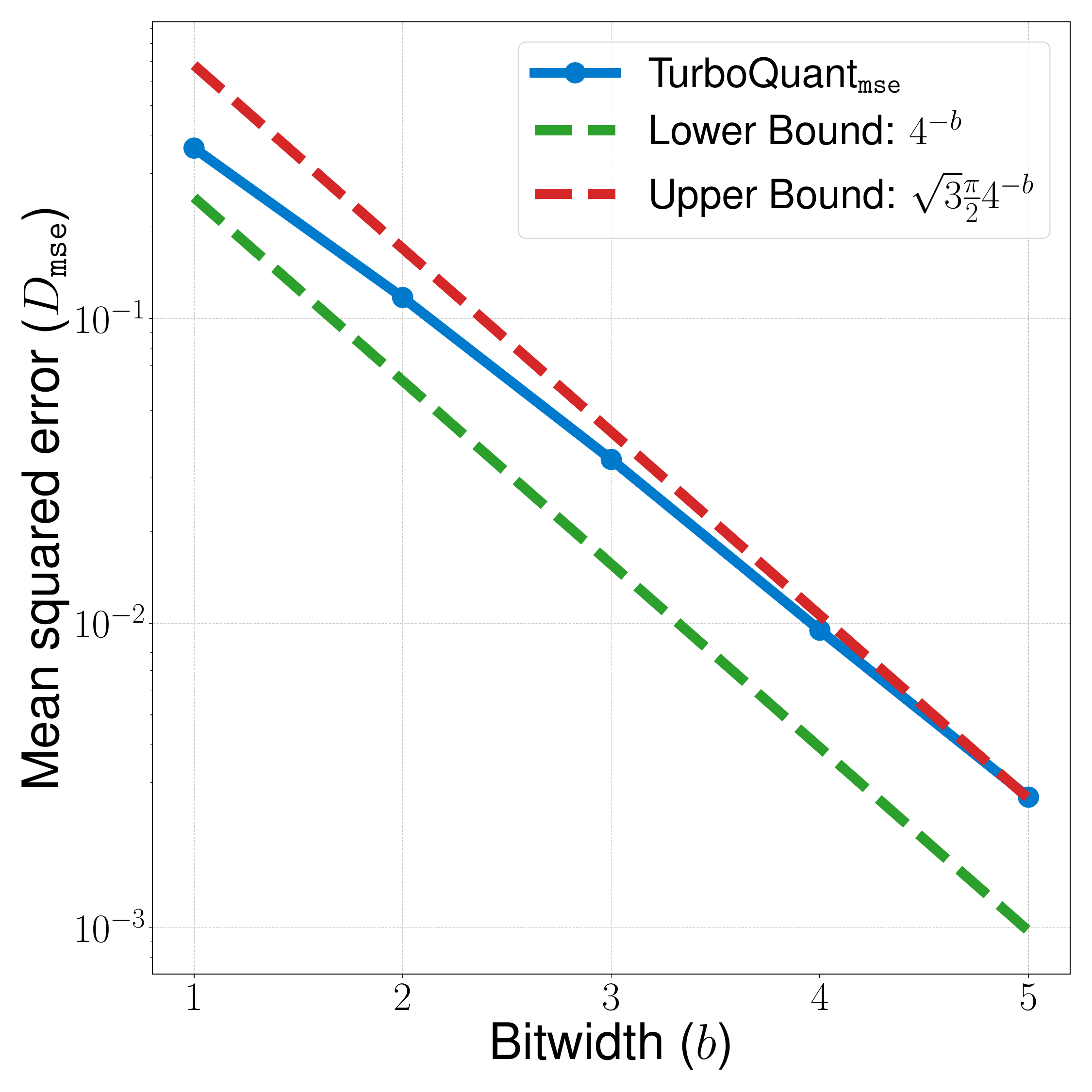}
    \end{subfigure}
    \label{fig:lower_upper}
    \end{tabular}
    \caption{Comparison of inner-product error and MSE against theoretical bounds across different bit ratios.}
    \end{center}
\end{figure}

Along with the histograms, we also plot \cref{fig:lower_upper} the average inner product error and MSE between the original and quantized vectors across different bit ratios. These plots are drawn alongside the upper and lower bounds established in our theoretical analysis. Our observations confirm that the results align with the theoretical predictions. Specifically, for inner product estimation, the $\TurboQuant\text{\tt prod}$ approach performs better at lower bit ratios. However, as the bit count increases, $\TurboQuant\text{\tt mse}$ reduces bias and ultimately achieves superior performance in inner product estimation.

\subsection{Needle-In-A-Haystack} \label{sec:niah}

\begin{figure}[h]
    \centering
    \begin{tabular}{ccc}
        \begin{minipage}{0.33\textwidth}
            \centering
            \textbf{SnapKV} \\ Score: 0.858 \\ 
            \includegraphics[width=\textwidth]{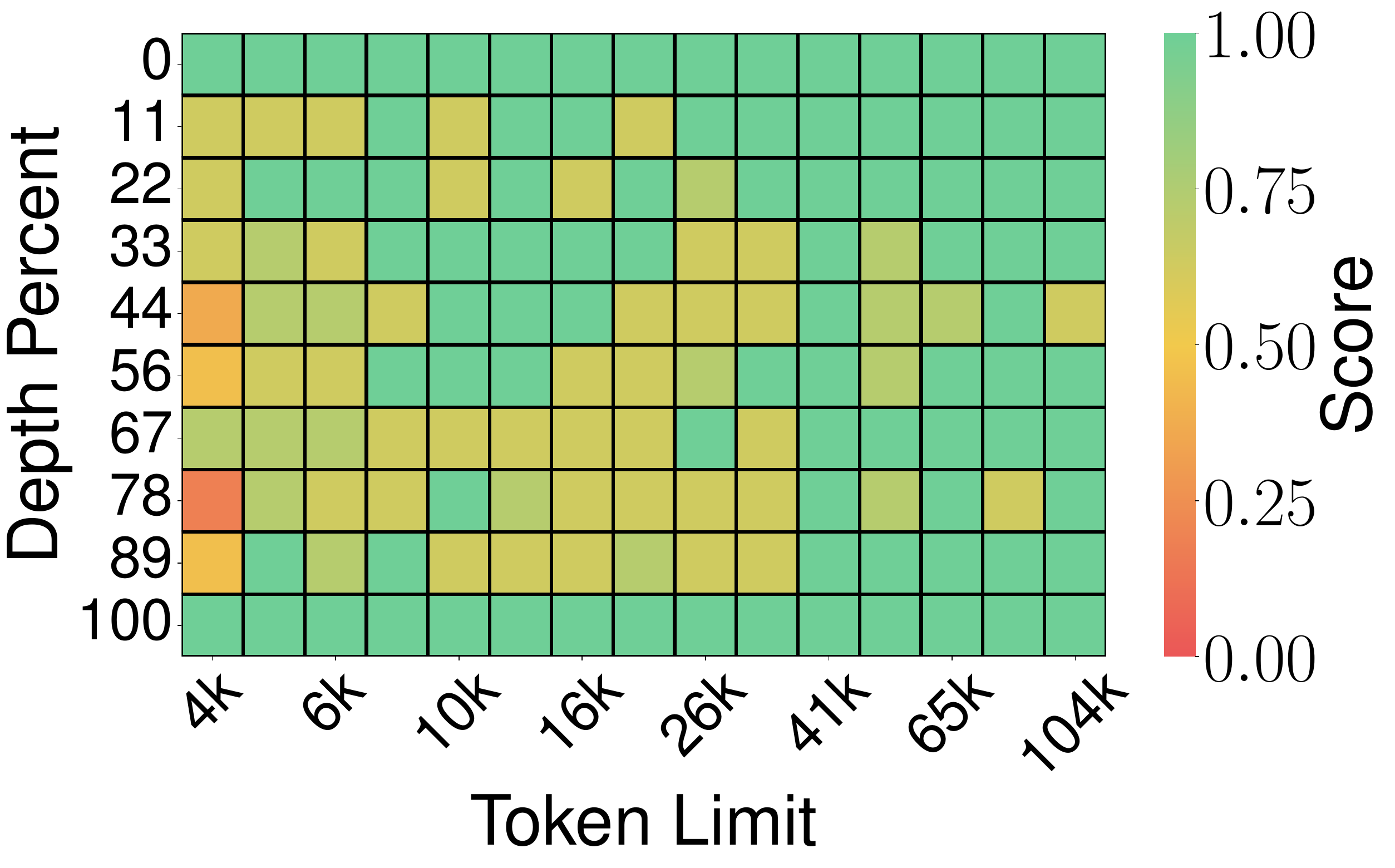}
        \end{minipage} &
        \begin{minipage}{0.33\textwidth}
            \centering
            \textbf{PyramidKV} \\ Score: 0.895 \\ 
            \includegraphics[width=\textwidth]{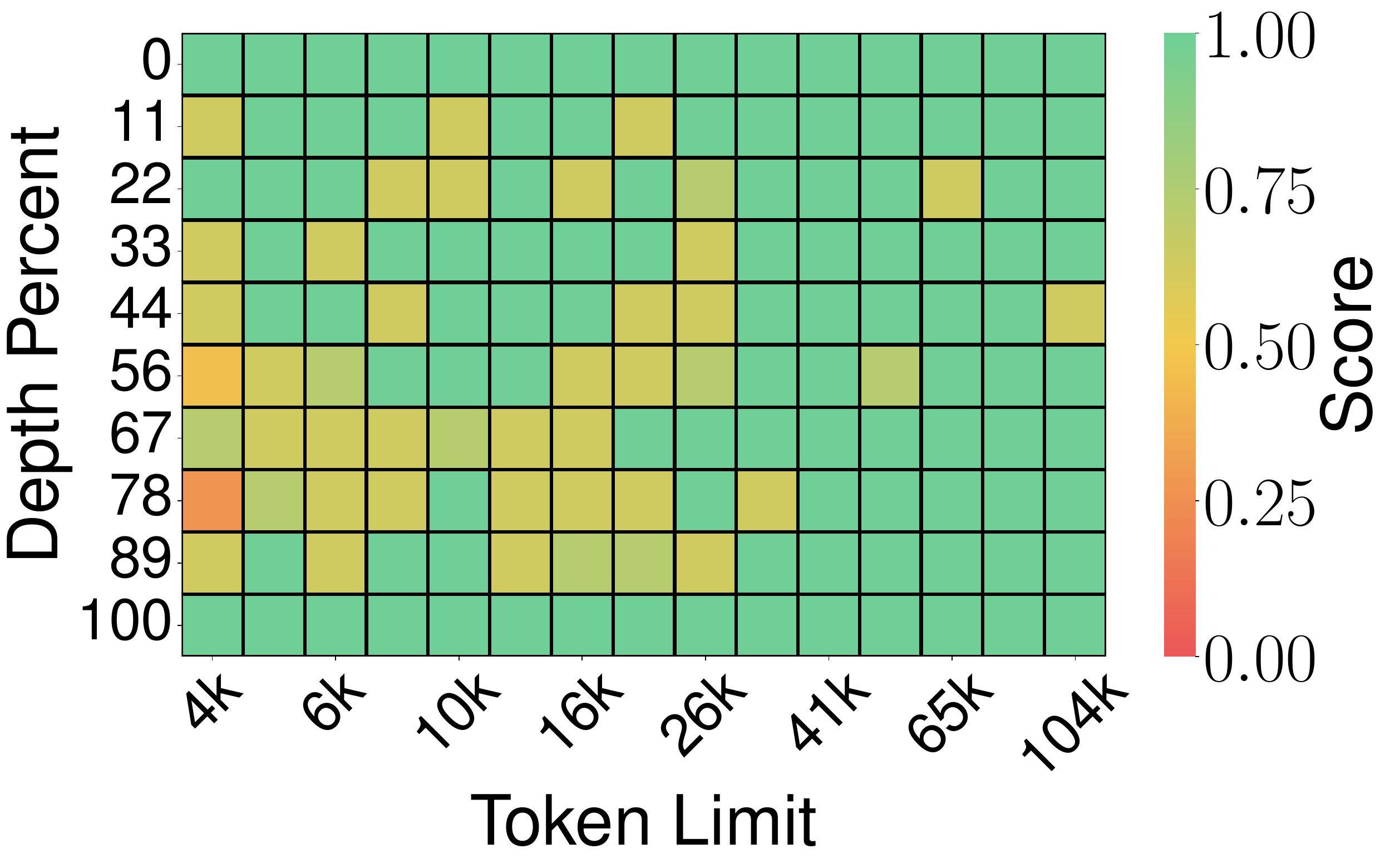}
        \end{minipage} &
        \begin{minipage}{0.33\textwidth}
            \centering
            \textbf{KIVI} \\ Score: 0.981 \\ 
            \includegraphics[width=\textwidth]{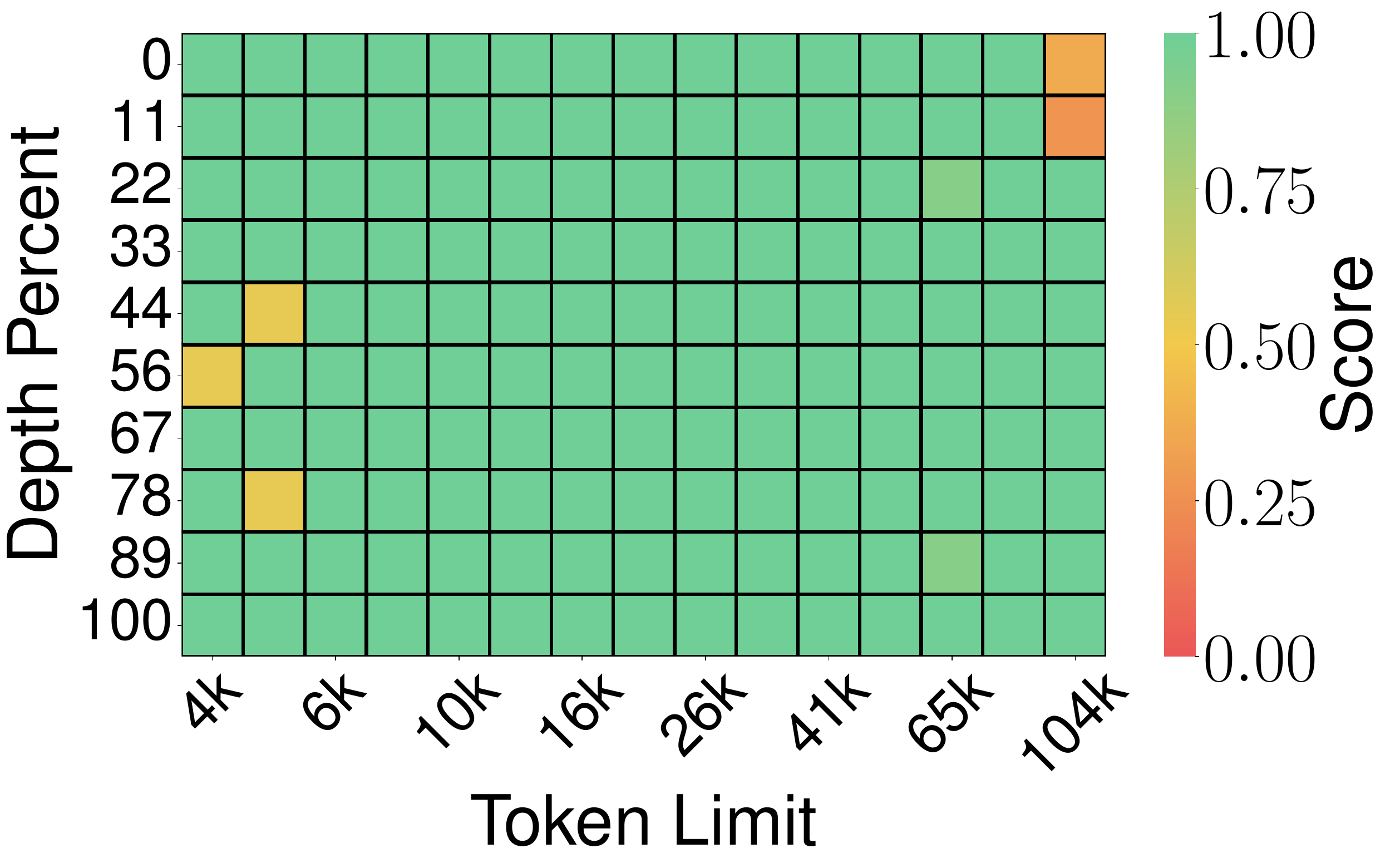}
        \end{minipage}\vspace{0.5cm} \\
        \begin{minipage}{0.33\textwidth}
            \centering
            \textbf{PolarQuant} \\ Score: 0.995 \\ 
            \includegraphics[width=\textwidth]{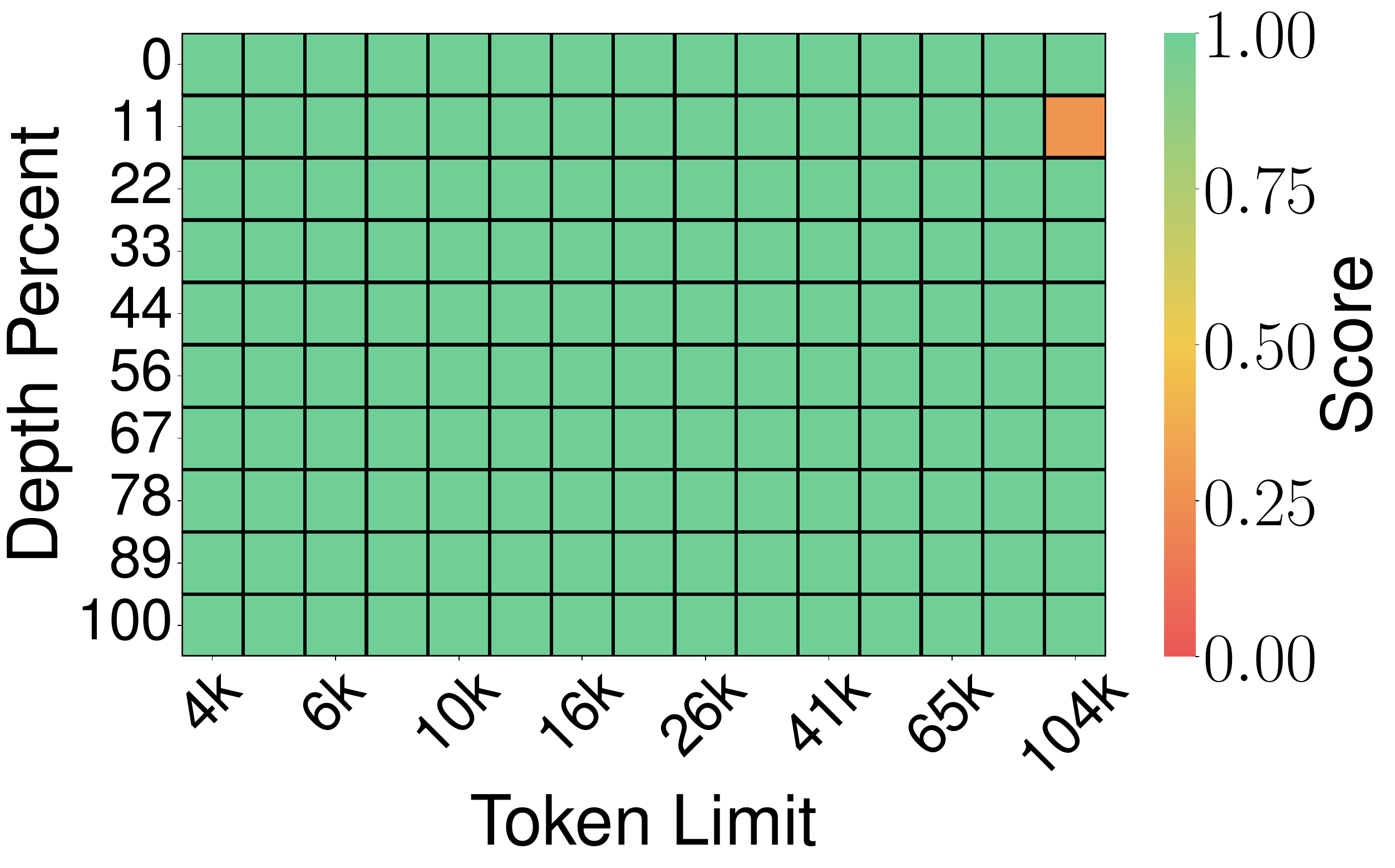}
        \end{minipage} &
        \begin{minipage}{0.33\textwidth}
            \centering
            \textbf{Full-Precision} \\ Score: 0.997 \\ 
            \includegraphics[width=\textwidth]{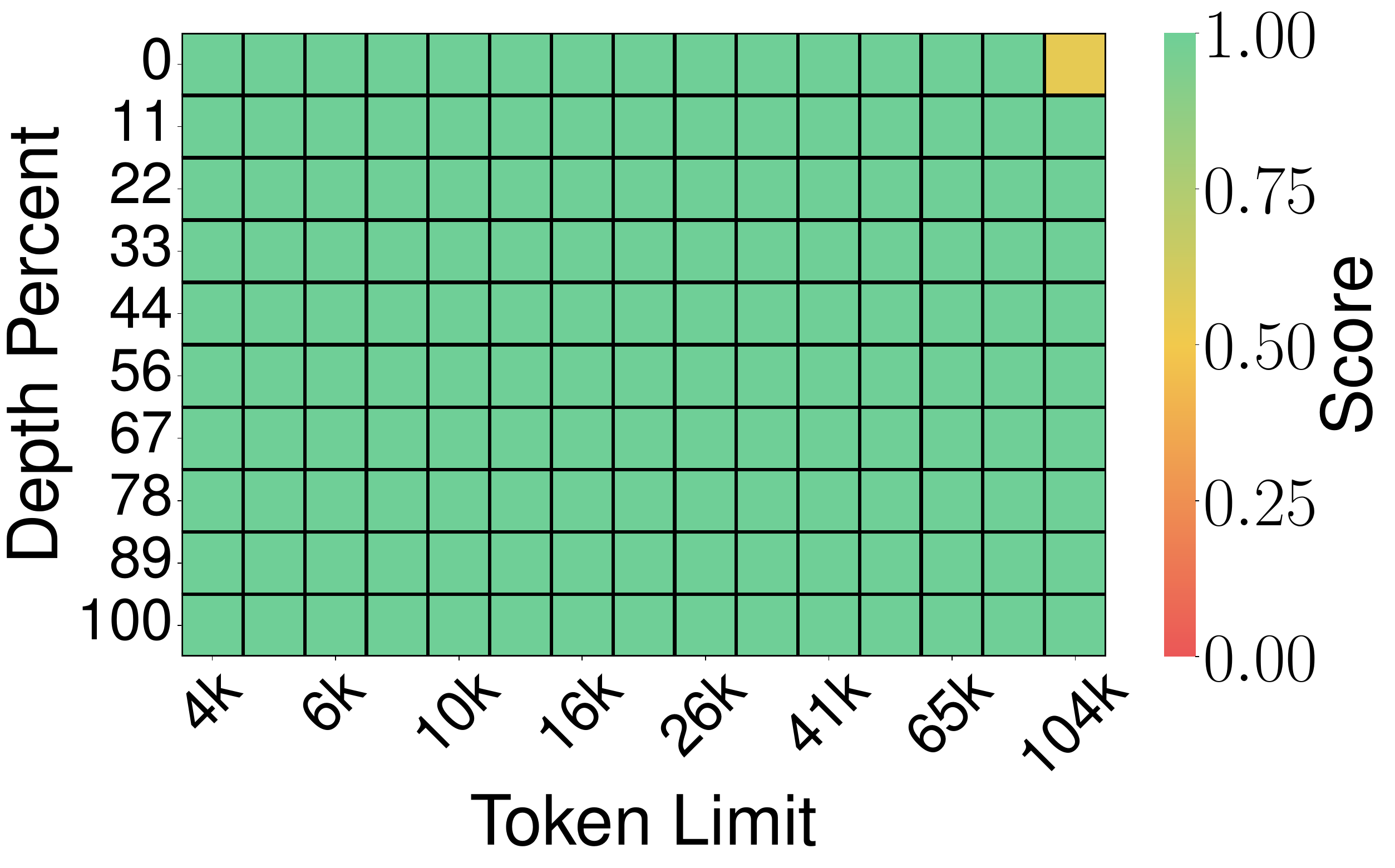}
        \end{minipage} &
        \begin{minipage}{0.33\textwidth}
            \centering
            \textbf{\TurboQuant} \\ Score: 0.997 \\ 
            \includegraphics[width=\textwidth]{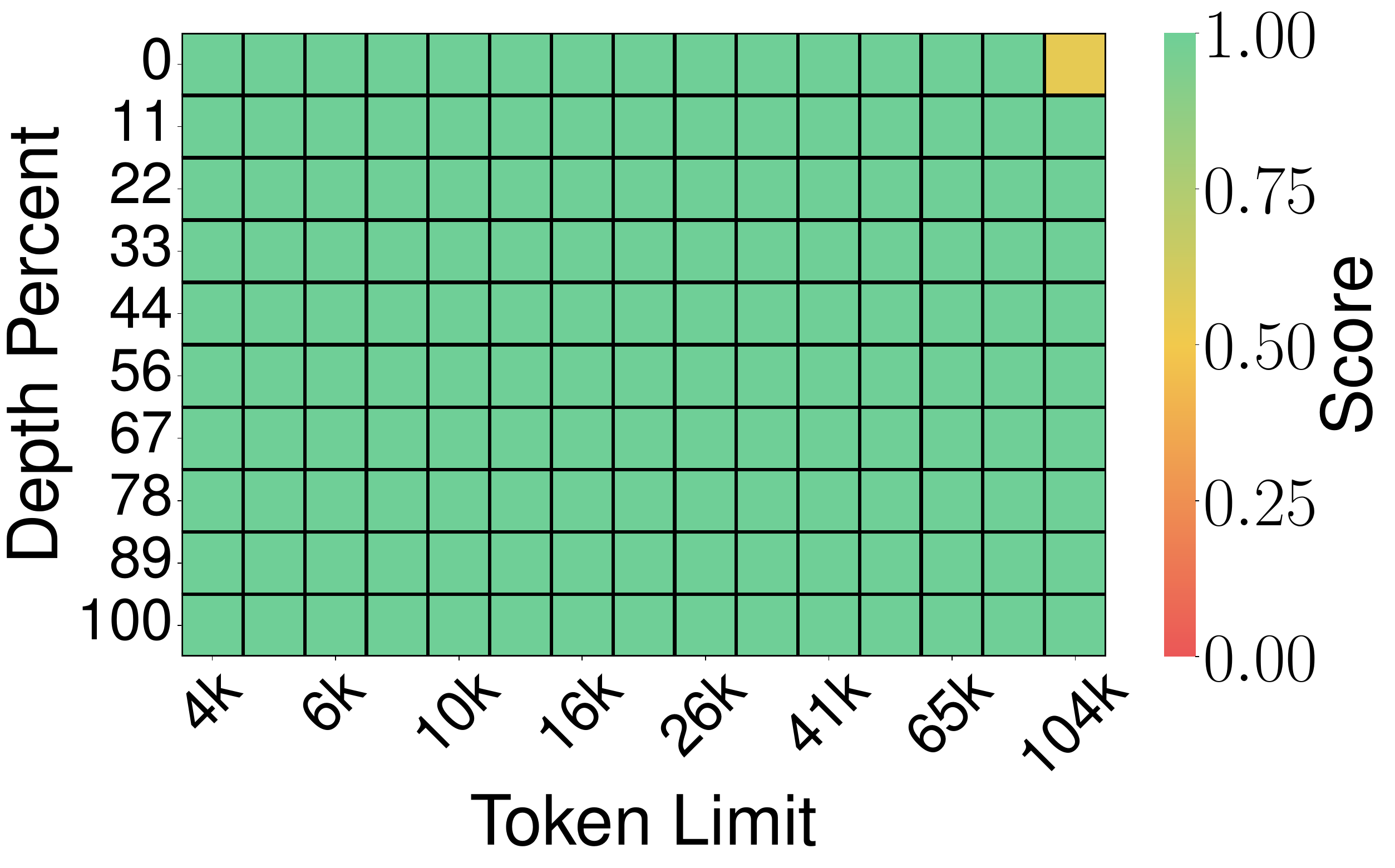}
        \end{minipage} \\
    \end{tabular}
    \caption{Evaluation of \texttt{Llama-3.1-8B-Instruct} on the ``Needle-In-A-Haystack" test, where a model must retrieve a hidden sentence from long-context sequences. While some methods struggle with recall, \TurboQuant, despite being more than \textbf{$4 \times$} quantized, achieves the same exact performance as the uncompressed baseline.}
    \label{fig:needle_comparison}
\end{figure}

The ``Needle-In-A-Haystack Test"" \cite{niah} is a benchmark designed to evaluate a model’s ability to retrieve specific information embedded within a long document. The test involves placing a unique sentence (the "needle") at an arbitrary location within a much larger text (the "haystack") and assessing whether the model can successfully extract it.

Following the experimental setup of \citet{fu2024data}, we conduct evaluations using the \llama{} model. To analyze performance across different input sequence lengths, we vary the document size from \emph{4k to 104k tokens}. The primary metric used for evaluation is the \emph{recall score}, which measures how accurately the model retrieves the hidden sentence.

For comparison, we benchmark our approach against several state-of-the-art memory-efficient methods, including PolarQuant~\cite{han2025polarquant}, SnapKV~\cite{li2024snapkv}, PyramidKV~\cite{cai2024pyramidkv}, and KIVI~\cite{liu2024kivi}. Each method is tested under a memory compression ratio of 0.25, meaning that only 25\% of the full KV cache is utilized.

The results, illustrated in \cref{fig:needle_comparison}, reveal that quantization methods with theoretical guarantees, such as PolarQuant and \TurboQuant, outperform token-level compression techniques like SnapKV and PyramidKV, as well as scalar quantization approaches like KIVI, which lack formal theoretical guarantees. Notably, \TurboQuant achieves identical performance to the full-precision model, even at $4 \times$ compression, making it a robust solution for long-context processing.


\subsection{End-to-end Generation on LongBench}
\label{sec:exp_kv_end2end}

We experiment with various KV cache compression algorithms on the LongBench dataset~\cite{bai2023longbench}, which encompasses a broad range of long-text scenarios, including single- and multi-document question-answering, summarization, few-shot learning, synthetic tasks, and code completion. To ensure a balanced evaluation across different context lengths, we employ \textbf{LongBench-E}, a subset designed with a more uniform length distribution. This enables a fair assessment of each model’s performance across varying context sizes, making it a more reliable benchmark for evaluating compression techniques.

We compare \TurboQuant against the leading baseline methods introduced in \cref{sec:niah}, using both \llama{} and \mistral{}. Unlike existing approaches such as \textbf{KIVI} and \textbf{PolarQuant}, which leave generated tokens unquantized, our method applies quantization even during the streaming generation process.

As shown in \cref{tab:performance_comparison}, our approach outperforms other methods for both \llama{} and \mistral{}, achieving significantly higher average scores. We evaluate our method using \textbf{2.5-bit} and \textbf{3.5-bit} quantization during text generation. 
These non-integer bit precisions result from our strategy of splitting channels into outlier and non-outlier sets, and applying two independent instances of \TurboQuant to each, allocating higher bit precision to outliers. This outlier treatment strategy is consistent with prior work \cite{zandieh2024qjl, su2025rotatekvaccuraterobust2bit} . 
For example, in our 2.5-bit setup, 32 outlier channels are quantized at 3 bits, while the remaining 96 channels use 2 bits, leading to an effective bit precision of \( (32 \times 3 + 96 \times 2) / 128 = 2.5 \). For 3.5-bit quantization, a different ratio of outliers and regular channels leads to a higher effective bit precision. 
Despite using fewer bits than competing techniques, \TurboQuant maintains performance comparable to unquantized models. Remarkably, we achieve this while compressing quantized vectors by at least a factor of $4.5 \times$.



\setlength{\tabcolsep}{12pt}
\def\arraystretch{0.8}%
\begin{table}[t]
\centering
\scalebox{0.9}{
\small
\setlength{\tabcolsep}{4pt}
\begin{tabular}{@{}lcccccccc@{}}
        \toprule 
        \textbf{Method} & \textbf{KV Size} & \textbf{SingleQA} & \textbf{MultiQA} & \textbf{Summarization} & \textbf{Few shot} & \textbf{Synthetic} & \textbf{Code} &\textbf{Average} \\ 
        \midrule
        \multicolumn{9}{c}{\llama{}} \\
        Full Cache    & $16$ & $45.29$ & $45.16$    & $26.55$    & $68.38$   & $59.54$    & $46.28$ & ${50.06}$ \\
        \midrule
        \midrule
        KIVI & $3$ & $43.38$ & $37.99$ & $27.16$    & $68.38$   & $59.50$    & $44.68$ & $48.50$ \\ 
        \midrule
        KIVI & $5$ & $45.04$ & $45.70$ & $26.47$    & $68.57$   & $59.55$    & $46.41$ & $50.16$ \\ 
        \midrule
        \midrule
        PolarQuant & $3.9$ & $45.18$ & $44.48$ & $26.23$    & $68.25$   & $60.07$    & $45.24$ & $49.78$ \\ 
        \midrule
        \midrule
        \TurboQuant (ours) & $2.5$ & $44.16$ & $44.96$ & $24.80$    & $68.01$   & $59.65$    & $45.76$ & $49.44$ \\ 
        \midrule
        \TurboQuant (ours) & $3.5$ & $45.01$ & $45.31$ & $26.00$    & $68.63$   & $59.95$    & $46.17$ & $50.06$ \\ 
        \midrule
        \multicolumn{9}{c}{\mistral{}} \\
        \midrule
        Full Cache    & $16$ & $47.53$ & $49.06$    & $26.09$    & $66.83$   & $53.50$    & $47.90$ & ${49.89}$ \\
        \midrule
        \midrule
        \TurboQuant (ours) & $2.5$ & $48.38$ & $49.22$ & $24.91$    & $66.69$   & $53.17$  & $46.83$ & $49.62$ \\ 
    \end{tabular}
}
\caption{LongBench-V1~\cite{bai2023longbench} results of various KV cache compression methods on \llama{}.}
\label{tab:performance_comparison}
\end{table}

\subsection{Near Neighbour Search Experiments}
\label{sec:nn_exp}
In this section, we establish the strength of our proposed method, even in the context of near-neighbor search. We conduct our experiments using the DBpedia \cite{thakur2021beir} Entities dataset, which has been encoded into 1536-dimensional\footnote{https://huggingface.co/datasets/Qdrant/dbpedia-entities-openai3-text-embedding-3-large-1536-1M} and 3072-dimensional \footnote{https://huggingface.co/datasets/Qdrant/dbpedia-entities-openai3-text-embedding-3-large-3072-1M} spaces using OpenAI3 embeddings. Additionally, we evaluate performance on a lower-dimensional dataset, utilizing the standard GloVe \cite{pennington2014glove} embeddings.
To construct our experimental setup, we randomly sample 100,000 data points from the dataset, denoted as training set, which serves as our primary training and evaluation set. Furthermore, we extract 1,000 distinct entries, denoted as query set, to be used as query points for datasets that do not explicitly provide a query set. For the GloVe dataset, we use a pre-existing query set consisting of 10,000 points.

We compare our method, \TurboQuant, against two baseline quantization approaches: Product Quantization (PQ) and RabitQ \cite{gao2024practical}. To ensure a fair comparison, we quantize the dataset training set using all three methods and evaluate their performance based on recall ratio at top-k, denoted as 1@k. Specifically, this metric assesses how often the true top inner product result is captured within the top-k approximated results returned by each algorithm.

\begin{table}[ht]
    \centering
    \begin{tabular}{lccc}
        \toprule
        Approach & d=200 & d=1536 & d=3072 \\
        \midrule
        Product Quantization & 37.04 & 239.75 & 494.42 \\
        RabitQ & 597.25 & 2267.59 & 3957.19 \\
        \TurboQuant & 0.0007 & 0.0013 & 0.0021 \\
        \bottomrule
    \end{tabular}
    \caption{Quantization time (in seconds) for different approaches across various dimensions using 4-bit quantization.}
    \label{tab:quantizing_time}
\end{table}

\textbf{Product Quantization (PQ)} relies on the k-means algorithm to construct codebooks, which require separate storage. 
As the number of bits increases, the size of the codebook grows exponentially, leading to additional storage overhead. In our experiments, we carefully tuned the parameters to match the bit allocation of other methods. 
The most efficient implementation, designed for rapid querying, employs AVX2 In-Register Lookup Tables (LUTs). 
Specifically, it uses LUT16 with (l = 16) codewords. However, we observed substantial quality degradation at this configuration. 
To achieve a balance between speed and accuracy, we opted for a version of PQ that uses LUT256, which contains 256 codewords. 
For 2-bit quantization, it groups 4 coordinates per lookup, while for 4-bit quantization, it groups 2 coordinates per lookup. 
Notably, since we use the same dataset for both training and evaluation, PQ benefits from an inherent advantage in this setup.

\textbf{RabitQ.} 
Unlike PQ, RabitQ lacks a fully vectorized implementation, making it impossible to leverage GPU acceleration. As a result, it runs significantly slower on CPU. Additionally, the method incurs extra computational overheads that we do not explicitly account for in the bit ratio comparisons. While RabitQ claims a certain bit ratio, in practice, it utilizes more bits than reported due to these inefficiencies.

Despite the advantages granted to the baseline methods, \TurboQuant consistently outperforms both Product Quantization and RabitQ in terms of recall ratio across all experiments. This demonstrates the robustness and efficiency of our approach, making it a compelling alternative for high-dimensional quantization-based search tasks.

\begin{figure}[th]
    \centering
    \begin{tabular}{ccc}
        \begin{subfigure}{0.3\textwidth}
            \centering
            \caption{GloVe - d=200}
            \centering
            \includegraphics[width=\textwidth]{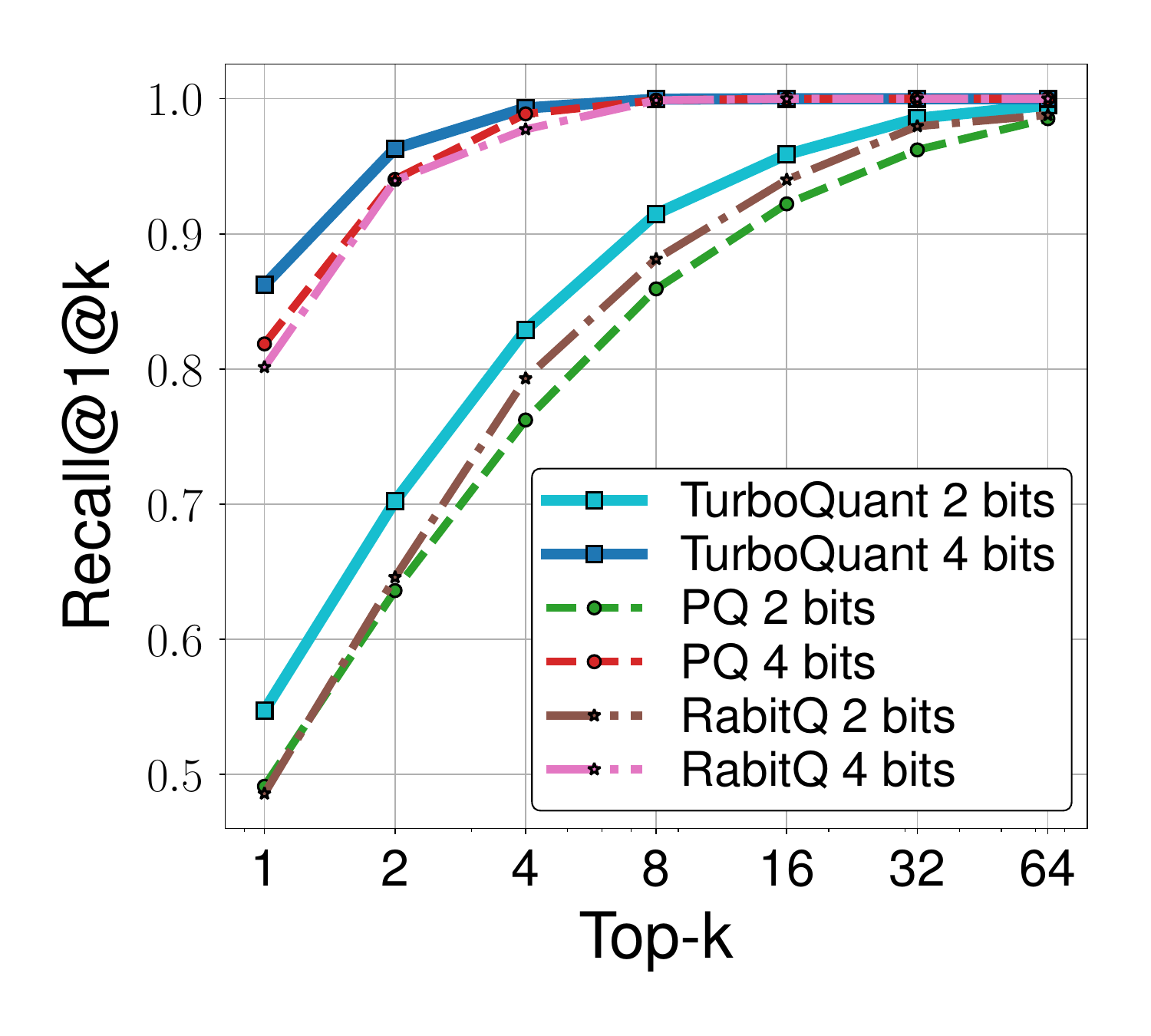}
        \end{subfigure}&
        \begin{subfigure}{0.3\textwidth}
        \centering
        \caption{OpenAI3 - d=1536}
            \centering
            \includegraphics[width=\textwidth]{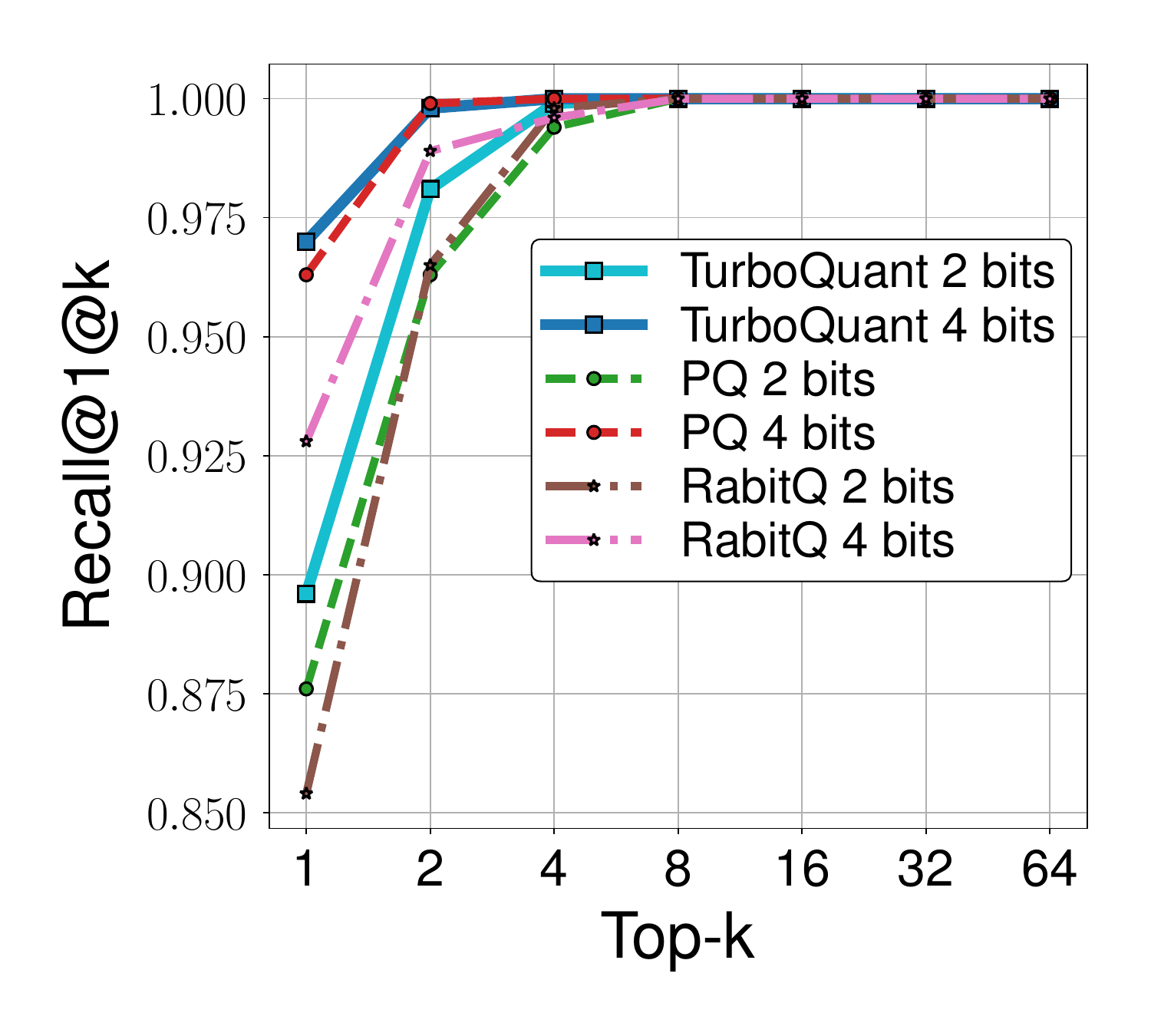}
        \end{subfigure}&
        \begin{subfigure}{0.3\textwidth}
        \centering
        \caption{OpenAI3 - d=3072}
            \centering
            \includegraphics[width=\textwidth]{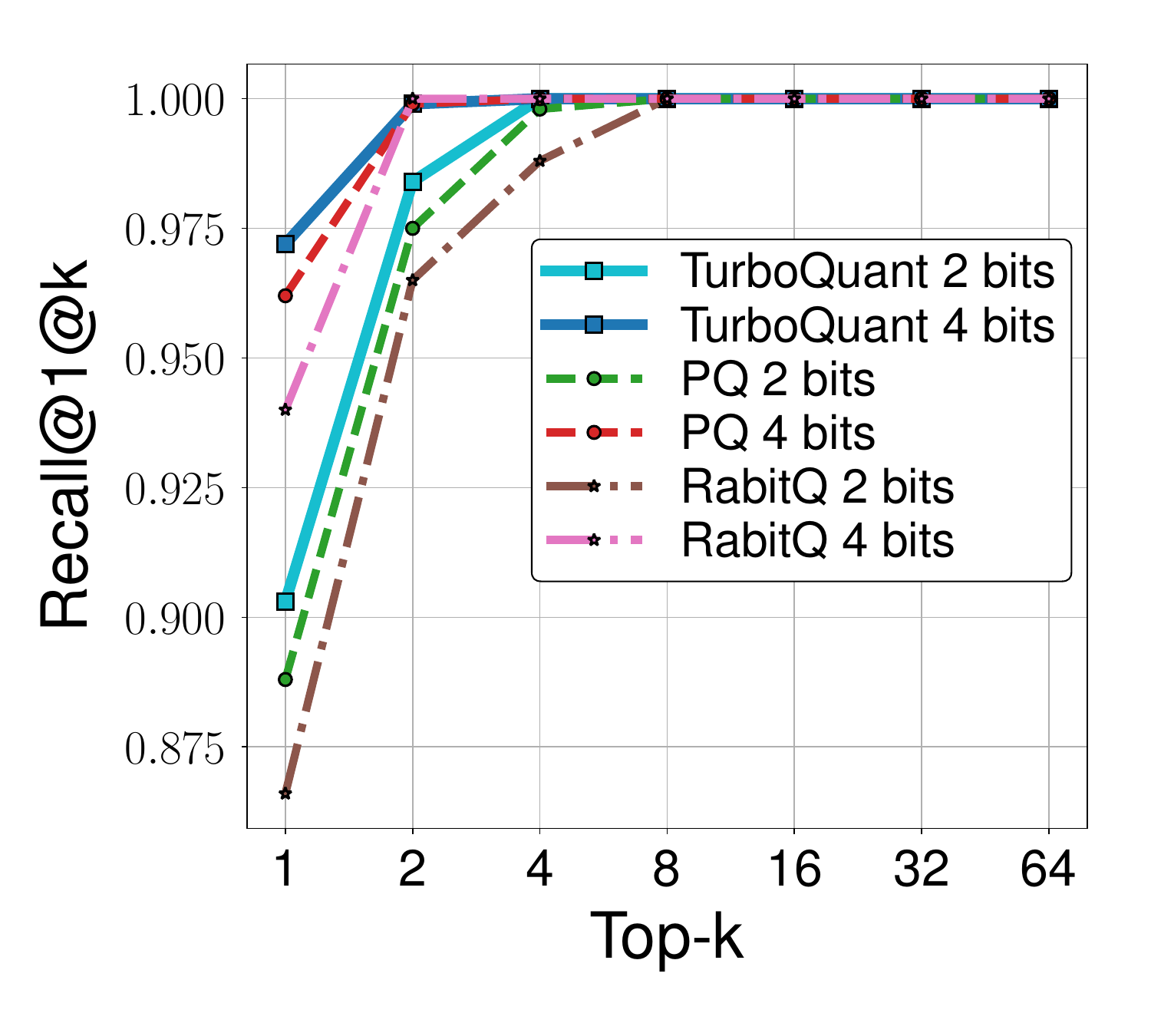}
    \end{subfigure}
    \end{tabular}
    \caption{Recall comparison on different datasets with different embedding dimensions.}
\end{figure}

\bibliography{paper}
\bibliographystyle{icml2025}

\newpage
\appendix
\onecolumn


\end{document}